\documentclass[accepted]{ci4ts2024} % after acceptance, for a revised version; 
\usepackage[american]{babel}
\usepackage{amsmath}
\usepackage{amssymb}
\usepackage{mathtools}
\usepackage{amsthm}
\usepackage{hyperref}
\usepackage{url}
\usepackage{duckuments}

\usepackage{graphicx}
\usepackage{subcaption}
\usepackage[utf8]{inputenc} % allow utf-8 input
\usepackage[T1]{fontenc}    % use 8-bit T1 fonts

\usepackage{booktabs}       % professional-quality tables
\usepackage{amsfonts}       % blackboard math symbols
\usepackage{nicefrac}       % compact symbols for 1/2, etc.
\usepackage{microtype}      % microtypography
\usepackage{color}
\usepackage{lipsum}
\usepackage{booktabs} % commands to create good-looking tables
\usepackage{tabularx,longtable,multirow}%hangcaption
\usepackage{makecell}
\newcommand{\A}{\mathcal{A}}
\newcommand{\R}{\mathbb{R}}

\newcommand{\x}{\bm{x}}

\newcommand{\y}{\mathbf{y}}
\newcommand{\z}{\bm{z}}

\newcommand{\jointfamily}{\mathcal{P}^{M,T}_{\A,\B}}
\newcommand{\msmfamily}{\mathcal{M}^T(\Pi_{\A}^M,\mathcal{P}_{\B}^M)}

\newcommand{\B}{\mathcal{B}}

\newcommand{\p}{\mathbf{p}}

\newcommand{\s}{\mathbf{s}}

\newcommand{\mparam}{\bm{\theta}}	% model param
\def\1{\bm{1}}
\usepackage{bm}
\usepackage{enumitem}
\usepackage{wrapfig}
\usepackage{xcolor}
\usepackage{natbib}
\usepackage[capitalize,noabbrev]{cleveref}
\DeclareMathOperator*{\argmax}{arg\,max}

\theoremstyle{plain}
\newtheorem{theorem}{Theorem}[section]
\newtheorem{proposition}[theorem]{Proposition}
\newtheorem{lemma}[theorem]{Lemma}
\newtheorem{corollary}[theorem]{Corollary}
\theoremstyle{definition}
\newtheorem{definition}[theorem]{Definition}

\theoremstyle{remark}

\newlist{thmlist}{enumerate}{1}
\setlist[thmlist]{label=(\roman{thmlisti}),
                  ref=\thetheorem.(\roman{thmlisti}),
                  noitemsep}

\title{Identifying Nonstationary Causal Structures with \\  High-Order Markov Switching Models}

\author[1]{\href{mailto:<cb221@imperial.ac.uk>?Subject=Your UAI 2024 paper}{Carles Balsells-Rodas}}
\author[2]{Yixin Wang}
\author[1,3]{Pedro A. M. Mediano}
\author[1]{Yingzhen Li}
\affil[1]{%
    Imperial College London
}
\affil[2]{%
    University of Michigan
}
\affil[3]{%
    Division of Psychology and Language Sciences, University College London
  }
  
\begin{document}
\maketitle

\begin{abstract}
Causal discovery in time series is a rapidly evolving field with a wide variety of applications in other areas such as climate science and neuroscience. Traditional approaches assume a stationary causal graph, which can be adapted to nonstationary time series with time-dependent effects or heterogeneous noise. In this work we address nonstationarity via regime-dependent causal structures. We first establish identifiability for high-order Markov Switching Models, which provide the foundations for identifiable regime-dependent causal discovery. Our empirical studies demonstrate the scalability of our proposed approach for high-order regime-dependent structure estimation, and we illustrate its applicability on brain activity data.
\end{abstract}
\vspace{-.5em}
\section{Introduction}

Identifying causal relationships from observational data is a challenging problem \citep{spirtes2000causation,pearl2009causal}, particularly in the time series domain where temporal dependencies must be considered \citep{assaad2022survey}. Traditional approaches assume time-invariant structures \citep{granger1969investigating,peters2013causal}, where identifiability extends from non-temporal causal discovery \citep{peters2017elements}. Causal discovery encompasses a diverse taxonomy \citep{glymour2019review}. Constraint-based methods use conditional independence test to recover the underlying graph structure, such as PC \citep{spirtes2000causation}, FCI \citep{spirtes2001anytime}, and extensions to time series \citep{runge2018causal}.  Score-based methods optimise score functions from causal graphs \citep{chickering2002optimal}; for time series, Dynotears \citep{pamfil2020dynotears} relaxes the combinatorial search using constraint-based optimisation from \citet{zheng2018dags}. Functional causal model-based methods represent effects as functions of their causes via structural equation models \citep{shimizu2006linear,peters2014causal}. In time series, these approaches fit a constrained dynamic model which favours identifiability \citep{peters2013causal}. 

Recent developments explore nonstationary time series using time-dependent effects \citep{huang2015identification,huang2019causal}, heterogeneous/history-dependent noise \citep{huang2020causal,gong2023rhino}, or leveraging contextual information across datasets \citep{gunther2023causal}. Alternatively, nonstationarity can be addressed by assuming time-dependent causal structures, which is referred to as \emph{causal nonstationarity} \citep{assaad2022survey}. Regime-dependent causal discovery \citep{saggioro2020reconstructing} assumes discrete latent states (regimes) that alter the temporal dependencies, making the causal structure stationary within each regime. However, structure identifiability in this domain remains unexplored.

In this work, we establish identifiable regime-dependent causal discovery using high-order Markov Switching Models (MSMs; \citep{hamilton1989new}). They extend Hidden Markov Models (HMMs; \citep{baum1966statistical}) by incorporating autorregressive dependencies on the observations given discrete latent variables. Identifiability of stationary HMMs \citep{gassiat2016inference} is present in recent ICA literature \citep{halva2020hidden,halva2021disentangling}. However, extensions to other structured priors, such as MSMs, cannot include transition identifiability due to the requirement of independent mixture components \citep{allman2009identifiability}. 

Recently, \citet{balsells-rodas2024on} established identifiability for first-order MSMs. We generalise this to high-order dependencies by addressing the challenge of increasing additional overlapping variables in the joint distribution, which requires strengthening the assumptions in the non-parametric case. We then present parametrisations via conditional Gaussians with analytic moments (Section \ref{sec:msm_theory}). This foundational result enables identifiability of regime-dependent causal structures due to the accessibility of transition derivatives, which we present in Section \ref{sec:causal_discovery}. Finally, we demonstrate the scalability of our identifiable MSM to high-order structure estimation, and show applications to neuroscience with brain activity data (Section \ref{sec:exp}).
%Our contributions are as follows:
%\vspace{-.25em}
%\begin{itemize}
%    \item We establish identifiability of high-order MSMs, where transition distributions are parametrised by conditional Gaussians with analytic moments (Sec. \ref{sec:msm_theory}). 
%    \item We formalise regime-dependent causal discovery using high-order MSMs, and show identifiability of the regime-dependent causal structure (Sec. \ref{sec:causal_discovery}).
%    \item We demonstrate the scalability of our identifiable MSM to high-order structure estimation, and show applications to neuroscience with brain activity data (Sec. \ref{sec:exp}).
%\end{itemize}
\vspace{-.5em}
\section{Background}\label{sec:background}
%\subsection{Identifiability of Deep Generative Models}

%\begin{equation}
%    p_{\mparam}(\x) = p_{\tilde{\mparam}}(\x) \iff \mparam \sim \tilde{\mparam}
%\end{equation}
%where here $\sim$ denotes an equivalence class which establishes connections between $\mparam$ and $\tilde{\mparam}$. Some examples include permutation equivalence often used in mixture model literature \citep{yakowitz1968identifiability}, or even affine transformation equivalence used more recently \citep{kivva2022identifiability}.
%In this work, we provide identifiability analysis with functions parameterised by neural networks. Therefore identifiability will instead refer to the functional form of the parameterisations only.

\paragraph{Causal Discovery in Time Series}
Causal discovery assumes a structural causal model (SCM; \citep{pearl2009causal}) underlying the data generation process. In this work we focus on structural equation models (SEMs; \citep{peters2017elements}), which generate data given a causal graph $\mathbf{G}\in \mathbb{G}$. Note $\mathbb{G}$ denotes a causal graph space. Consider a sequence of observations $\x_t\in\R^{d}$, with $t\in\{1,\dots,T\}$, where in the context of time series, the corresponding SEM at time $t$, with causal graph $\mathbf{G}$, is expressed as follows
\vspace{-.25em}
\begin{multline}
    x^{(j)}_t = f^{(j)}(\{x^{(i)}_{t-\tau} | x^{(i)}_{t-\tau} \in \mathbf{Pa}^{(j)}(\tau), \\ 0\leq \tau \leq M \}, \bm{\varepsilon}^{(j)}_{t}), \quad \bm{\varepsilon}^{(j)}_{t} \sim p_{\bm{\varepsilon}^{(j)}},
\end{multline}
for $j\in\{1,\dots,d\}$, where $p_{\bm{\varepsilon}^{(j)}}$ denotes the distribution of the independent noise at each time $t$. We also denote $\mathbf{Pa}^{(j)}(\tau)$ as the parents of variable $j$ with lag $\tau$, which are associated with the causal graph $\mathbf{G}$. We distinguish between \emph{instantaneous effects} ($\tau=0$) and \emph{lagged effects} ($\tau\geq 1$). Under no further assumptions, the causal graph $\mathbf{G}$ is unidentifiable from data. To exemplify, \citet{peters2013causal} proposes TiMINo, where the causal graph is identifiable if the SEM belongs to an \emph{identifiable functional model class}.

\vspace{-.25em}
\paragraph{Markov Switching Models}
Again denote $\x_t \in\R^d$ observations of a sequence with $t\in\{1,\dots,T\}$. Markov Switching Models (MSMs; \citep{hamilton1989new}), are a class of discrete latent variable models with latents $s_t\in\{1,\dots,K\}$. They consider autoregressive connections between the observed variables, conditioned on $s_t$. Given MSMs parametrised by $\mparam$, the likelihood $p_{\mparam}(\x_{1:T})$ can be computed exactly as follows:
\vspace{-.5em}
\begin{multline}\label{eq:msm}
   p_{\mparam}(\x_{1:T}) = \sum_{s_{M:T}} p_{\mparam}(\s_{M:T})  p_{\mparam}(\x_{1:M}| s_M) \\[-.75em]
   \prod_{t=M+1}^T p_{\mparam}(\x_t|\x_{t-1},\dots,\x_{t-M}, s_t)
\end{multline}
\vspace{-1.25em}

where we depict an MSM with lag $M$. For simplicity, we consider $s_M\in\{1,\dots, K_0\}$, and $s_t\in\{1,\dots, K\}$ for $M < t \leq T$. Although the prior $p_{\mparam}(\s_{M:T})$ in Eq (\ref{eq:msm}) is not specified, in our experiments we consider a first-order Markov chain. However, our theoretical results allow any form for $p_{\mparam}(\s_{M:T})$ as long as the states are independent of the observations $\x_{1:T}$.
\vspace{-.25em}
\section{Identifiability Analysis}

In this section, we first establish theoretical results for high-order Markov switching models. Then, we specify the identification of the corresponding regime-dependent causal structure. \citet{balsells-rodas2024on} already provides results for first-order MSMs using Gaussian transitions where the mean and covariance are parameterised by analytic functions. Our approach allows the incorporation of additional autoregressive dependencies without imposing further restrictions on the parametric family. Below we provide identifiability analysis with functions parameterised by neural networks. Due to neural network symmetries, identifiability will instead refer only to the functional form of the parameterisations.

\vspace{-.5em}
\subsection{High-Order MSMs}\label{sec:msm_theory}

We drop the subscript $\mparam$ for simplicity, and frame the MSM as a mixture of sequences assuming length $T < +\infty$. First, we define the family of initial and transition distributions with autoregressive order $M$:
\begin{equation}
\Pi_{\A}^M:= \Big\{p_a(\x_{1:M})|a\in \A \Big\},
\end{equation}
\begin{equation}
\mathcal{P}_{\B}^M := \Big\{ p_b(\x_t | \x_{t-1}, \dots,\x_{t-M}) | b \in \B \Big\},
\end{equation}
where $\A$ and $\B$ are index sets satisfying mild measure theoretic conditions (see Appendix \ref{app:non_parametric_mixtures}). This formulation connects with the notation presented in Eq. (\ref{eq:msm}) as we have for each $k_0\in\{1,\dots,K_0\}$: $p(\x_{1:M}|s_M=k_0) =p_a(\x_{1:M})$, for some  $a\in\A$; and for each $k\in\{1,\dots,K\}$: $p(\x_t| \x_{t-1:t-M}, s_t=k)=p_{b_t}(\x_t| \x_{t-1:t-M})$ for some $b_t\in\B$ with $M < t \leq T$. In other words, the families $\Pi_{\A}^M$ and $\mathcal{P}_{\B}^M$ contain infinitely many functions, and each the distributions in the MSM align with one of such functions.

Then, we can define a bijective \emph{path indexing} function $\varphi: \{1,\dots, C\}\rightarrow \{1,\dots,K_0\} \times \{1,\dots,K\}^{T-M}$, with a total number of mixture paths $C = K_0 \cdot K^{(T-M)}$, where $c_i=p(\s_{M:T}=\varphi(i))$ denotes the probability of a sequence of states $\s_{M:T}=\varphi(i)$. Under the following notation we introduce the family of MSMs of order $M$, which is an equivalent formulation of Eq.(\ref{eq:msm}) in the finite mixture sense.
\vspace{-1em}
\begin{multline}\label{eq:msm_finite_mixture}
    \msmfamily := \bigg\{ \sum_{i=1}^C c_i p_{a^i}(\x_{1:M}) \\[-.5em]
    \prod_{t=M+1}^T p_{b^i_t}(\x_t | \x_{t-1},\dots,\x_{t-M}) \Big| p_{a^i}\in\Pi^{M}_{\A}, p_{b^i_t}\in\mathcal{P}^{M}_{\B},\\ t>M,  a^i\in \A , b^i_t \in \B, (a^i,b^i_{M+1:T}) \neq (a^j,b^j_{M+1:T}),\\ \forall i\neq j, \sum_{i=1}^C c_i = 1, T > M \bigg\}.
\end{multline}
\vspace{-1.5em}

Given the requirement $(a^i,b^i_{M+1:T}) \neq (a^j,b^j_{M+1:T})$ for any $i\neq j$, we have an injective mapping $\phi : \{1,\dots, C\}\rightarrow \{1,\dots,K_0\} \times \{1,\dots,K\}^{T-M}$, such that $\phi(i) = (a^i,b^i_{M+1}, \dots, b^i_{T})$. Therefore, we can combine both indexing maps, $\phi \circ \varphi^{-1}$ to map a set of states $\s_{M:T}$ to some indices $a,b_{M+1:T}$, uniquely. 

The above notation allows the extension of MSM to finite mixtures, where we now have a mixture of $C$ trajectories indexed by some $(a,b_{M+1:T})$. Below we define the notion of identifiability of $\mathcal{M}^T(\Pi_{\A}^M, \mathcal{P}_{\B}^M)$.
\begin{definition}\label{def:identifiability}
We say the family of Markov switching models with order $M$, where $M\in\mathbb{Z}^+$, is \emph{identifiable up to permutations}, if for any $p,\tilde{p}\in \mathcal{M}^T(\Pi_{\A}^M, \mathcal{P}_{\B}^M)$,  
\begin{itemize}[leftmargin=0cm]
  \setlength{\itemsep}{1pt}
  \setlength{\parskip}{0pt}
  \setlength{\parsep}{0pt}
    \item[] $p(\x_{1:T}) = $
    \item[] $\quad\sum_{i=1}^{C} c_i p_{a^i}(\x_{1:M}) \prod_{t=M+1}^T p_{b^i_t}(\x_t | \x_{t-1},\dots,\x_{t-M}),$
    \vspace{.75em}
    \item[] $\tilde{p}(\x_{1:T}) = $
    \item[] $\quad\sum_{i=1}^{\tilde{C}} \tilde{c}_i p_{\tilde{a}^i}(\x_{1:M}) \prod_{t=M+1}^T p_{\tilde{b}^i_t}(\x_t | \x_{t-1},\dots,\x_{t-M}),$
\end{itemize}
with $p(\x_{1:T})=\tilde{p}(\x_{1:T}), \forall \x_{1:T}\in\R^{dT}$, we have $K_0=\tilde{K}_0$, $K=\tilde{K}$, %and there exist permutations $\sigma_1(\cdot)$, $\sigma_2(\cdot)$ such that:
and for each $1\leq i \leq C$ there is some $1 \leq j \leq \tilde{C}$ such that:
\begin{enumerate}[leftmargin=*]
    \item $c_i=\tilde{c}_j$;
    \item if $b^i_{t_1}=b^i_{t_2}$ for $t_1,t_2 > M$, $t_1\neq t_2$, then $\tilde{b}^j_{t_1}=\tilde{b}^j_{t_2}$;
    \item $p_{a^i}(\x_{1:M})=p_{\tilde{a}^j}(\x_{1:M}), \quad \forall\x_{1:M} \in \R^{dM}$;
    \item $p_{b^i_t}(\x_t | \x_{t-1},\dots,\x_{t-M})=p_{\tilde{b}^j_t}(\x_t | \x_{t-1},\dots,\x_{t-M})$,\small{$\forall(\x_t,\dots, \x_{t-M}) \in \R^{d(M+1)}$}.
    %\item $p_{a}(\x_{1:M})=p_{\sigma_1(a)}(\x_{1:M}), \quad \forall\x_{1:M} \in \R^{dM}, a\in\A$;
    %\item $p_{b}(\x_t | \x_{t-1:t-M})=p_{\sigma_2(b)}(\x_t | \x_{t-1:t-M}),\\ \forall(\x_t,\dots, \x_{t-M}) \in \R^{d(M+1)}, b\in\B$.
\end{enumerate}
\end{definition}
Given that the above framework is designed to explore identifiability in deep generative models, we focus on establishing identifiability for parametric models, and refer to Appendix \ref{app:proof_identifiability} for the non-parametric analyses. We formulate our parametric model in terms of a non-linear Gaussian transition given $M$ previous observations:
\begin{multline}\label{eq:gaussian_transition_family}
    \mathcal{G}_{\B}^M = \Big\{p_b(\x_t|\x_{t-1}, \dots, \x_{t-M}) = \\ \mathcal{N}(\x_t| \bm{m}(\x_{t-1}, \dots, \x_{t-M}, b), \bm{\Sigma}(\x_{t-1}, \dots, \x_{t-M}, b) \Big| \\ b\in\B, (\x_t,\dots,\x_{t-M})\in\R^{d(M+1)}\Big\},
\end{multline}
with $\bm{m}(\cdot,b): \R^{dM} \rightarrow \R^d$ and $\bm{\Sigma}(\cdot,b): \R^{dM} \rightarrow \R^{d\times d}$ being non-linear functions. Similarly to \citet{balsells-rodas2024on}, we define the following unique indexing assumption for the above Gaussian family:
\begin{multline}
\label{eq:unique_indexing_conditional_gaussian}
    \forall b \neq b' \in \B, \ \exists (\x_{t-1},\dots,\x_{t-M}) \in \mathbb{R}^{dM}, \ s.t. \\ \bm{m}(\x_{t-1},\dots,\x_{t-M}, b)
    \neq \bm{m}(\x_{t-1},\dots,\x_{t-M}, b') \\ \text{or} \ \bm{\Sigma}(\x_{t-1},\dots,\x_{t-M}, b) \neq \bm{\Sigma}(\x_{t-1},\dots,\x_{t-M}, b'),
\end{multline}
which assumes two Gaussian distributions on $\x_t$ for different $b,b'\in\B$, given $(\x_{t-1},\dots,\x_{t-M})$ . We also introduce a family of initial Gaussian distribution:
\begin{equation}
\label{eq:gaussian_initial_family}
    \mathcal{I}_{\A}^M := \Big\{p_a(\x_{1:M}) = \mathcal{N}(\x_{1:M}| \bm{\mu}(a), \bm{\Sigma}_1(a)) \ | \ a \in \A \Big\},
\end{equation}
with equivalent unique indexing assumptions,
\begin{equation}
\label{eq:unique_indexing_marginal_gaussian}
    a \neq a' \in \A \Leftrightarrow \bm{\mu}(a) \neq \bm{\mu}(a') \text{ or } \bm{\Sigma}_1(a) \neq \bm{\Sigma}_1(a').
\end{equation}
We now establish the following identifiability result for the above parametric MSM of order $M$. 
\begin{theorem}\label{thm:identifiability_main}
Consider the Markov switching model family with order $M$ defined in Equation (\ref{eq:msm_finite_mixture}) under non-linear Gaussian families $\mathcal{M}(\mathcal{I}_{\A}^M, \mathcal{G}_{\B}^M)$%, where $T = c\cdot M$, with $c,M\in\mathbb{Z}^+$
. Assume:
\begin{enumerate}[leftmargin=*]
    \item[(i)] Unique indexing for the transition family $\mathcal{G}_{\B}^M$, defined by Eq. (\ref{eq:unique_indexing_conditional_gaussian}); and the family $\mathcal{I}_{\A}^M$, defined by Eq. (\ref{eq:unique_indexing_marginal_gaussian});
    \item[(ii)] The transition distribution moments in $\mathcal{G}_{\B}^M$, $\bm{m}(\cdot,b): \R^{dM} \rightarrow \R^d$ and $\bm{\Sigma}(\cdot,b): \R^{dM} \rightarrow \R^{d\times d}$, are analytic functions, for any $b\in \B$;
\end{enumerate}
Then the Markov switching model family is identifiable as defined in \ref{def:identifiability}.
\end{theorem}
\emph{Proof sketch:} See Appendix \ref{app:proof_identifiability} for the full proof. The strategy can be summarised in the following steps.
\begin{enumerate}[leftmargin=*]
    \item From mixture models identifiability  \citep{yakowitz1968identifiability}, we show linear independence over the joint distribution family $\jointfamily:=\Pi_{\A}^M \otimes (\otimes^T_{m=M+1}\mathcal{P}_{\B}^M)$ is sufficient for MSM identifiability.
    \item We use similar induction techniques from \citet{balsells-rodas2024on} to show linear independence. We first find conditions for $T=2M$ and use induction on $T>2M$ (see further details below). We then solve the remaining cases ($T\not\equiv 0 \mod M$).
    \item We establish conditions for \emph{non-parametric} $\mathcal{P}_{\B}^M$ and $\Pi_{\A}^M$ such that the product $\Pi_{\A}^M\otimes(\otimes_{m=1}^M\mathcal{P}_{\B}^M) = \{p_{a^i}(\x_{1:M})\prod_{m=1}^{M} p_{b^i_{M+m}}(\x_{M+m} | \x_{m},\dots,\x_{M-1+m})\}$ contains linear independent functions (case $T=2M$).
    \item We show that the Gaussian families $\mathcal{I}_{\A}$ and $\mathcal{G}_{\B}$ under assumptions (i-ii) satisfy the \emph{non-parametric} case defined for $\Pi_{\A}^M$ and $\mathcal{P}_{\B}^M$ respectively.
\end{enumerate}
%\begin{remark}
%    Setting $T=c\cdot M$, with $c,M\in\mathbb{Z}^+$, i.e. $T$ being a multiple of $M$, can be relaxed to $T>M$ if we perform similar extensions of (a3; Lemma \ref{lemma:linear_independence_two_nonlinear_gaussians}) done for (a4) in Lemma \ref{lemma:extension_b4} (See Appendix \ref{app:extending_assumptions}). However, we leave this result for future work.
%\end{remark}

\paragraph{Non-parametric identifiability.}  The above identifiability result preserves the same model assumptions as the ones presented in \citet{balsells-rodas2024on}, which establish MSM identifiability for $M=1$. Key to their  technique is the establishment of linear independence from products of functions with one overlapping variable (Lemma B.1 in \citet{balsells-rodas2024on}). To exemplify, the following first-order joint distribution family
\begin{equation}\label{eq:joint_first_order}
    \Pi_{\A}^1 \otimes \mathcal{P}_{\B}^1 \otimes \mathcal{P}_{\B}^1 = \bigg\{ p_a(\textcolor{purple}{\x_{1}})p_{b_2}(\textcolor{teal}{\x_{2}}|\textcolor{purple}{\x_{1}})p_{b_3}(\x_3|\textcolor{teal}{\x_{2}})\bigg\},
\end{equation}
for every $a\in\A$ and $b_2,b_3\in\B$, contains exactly one overlapping variable for products of consecutive distributions (indicated in different colours).

Our strategy to generalise to any lag $M < T$ focuses on strengthening the assumptions in the non-parametric case, which are still satisfied when assuming analytic Gaussian transition models. The challenge for identifiability when increasing autoregressive dependencies is due to the increasing set of overlapping variables in the joint distribution family. This complicates proving its linear independence (\emph{Proof sketch}: step 1), as we no longer have the same overlapping structures as in Eq. (\ref{eq:joint_first_order}). Therefore, the induction technique used in \citet{balsells-rodas2024on} does not hold. To illustrate, for $M=3$, the joint distribution family has the following structure:
\vspace{-.5em}
\begin{multline}\label{eq:joint_product_3}
    \Pi^3_{\A} \otimes (\otimes^3 \mathcal{P}^3_{\B}) = \Big\{p_a(\textcolor{purple}{\x_{1}},\textcolor{teal}{\x_2},\textcolor{orange}{\x_3})p_{b_4}(\textcolor{violet}{\x_{4}}|\textcolor{purple}{\x_{1}},\textcolor{teal}{\x_2},\textcolor{orange}{\x_3}) \\ p_{b_5}(\textcolor{olive}{\x_{5}}|\textcolor{violet}{\x_{4}},\textcolor{orange}{\x_3},\textcolor{teal}{\x_2})p_{b_6}(\x_6|\textcolor{olive}{\x_{5}},\textcolor{violet}{\x_{4}},\textcolor{orange}{\x_3})\Big\},
\end{multline}
for every $a\in\A$ and $b_4,b_5,b_6\in\B$, where we indicate overlaping variables with different colours. 
%We note that for $M=1$, the overlap is only present between consecutive distribution. In the above equation, 
Compared to $M=1$ in Eq. (\ref{eq:joint_first_order}), when multiplying the initial and transition distribution, $p_a(\x_{1:3})p_{b_4}(\x_4|\x_{3},\x_2,\x_1)$, the overlap $\x_{1:3}$ shares some similarities. In fact, one can utilise the strategy in $M=1$ to show linear independence for $T=4$ in the above family (see Lemma \ref{lemma:linear_independence_two_nonlinear_gaussians}). Then for $T=5$, the conditioned variables in $p_{b_5}(\x_5|\x_{4},\x_3,\x_2)$ will overlap with both $p_a(\x_{1:3})$ $p_{b_4}(\x_4|\x_{3},\x_2,\x_1)$, and thus the previous strategy cannot be used to prove linear independence as the product does not directly reduce to Eq. (\ref{eq:joint_first_order}). Similar arguments apply for $T=6$, where $\x_3$ overlaps with all the functions in Eq. (\ref{eq:joint_product_3}).

To ensure linear independence of the joint distribution family for any lag $M$, we take an augmented variable approach, where we group variables to force overlaps only between consecutive distributions (similar to Eq. (\ref{eq:joint_first_order})). Therefore, the previous example with $M=3$ results as follows
\vspace{-.5em}
\begin{multline}
    \Pi^3_{\A} \otimes (\otimes^3 \mathcal{P}^3_{\B}) = \Big\{p_a(\textcolor{purple}{\x_{1:3}}) \\
    p_{b_{4:6}}(\x_{4:6}|\textcolor{purple}{\x_{1:3}}) |\quad a\in\A, b_{4:6}\in(\times^{3}\B)\Big\}.
\end{multline}
Now, similar techniques used in the case $M=1$ should also apply, provided that the product distribution $\{p_{b_{4:6}}(\x_{4:6}|\x_{1:3})\}$ satisfies the conditions (a4-a6) in Lemma \ref{lemma:linear_independence_two_nonlinear_gaussians}. In Appendix \ref{app:extending_assumptions}, we show that extending assumptions to product distributions (Lemmas \ref{lemma:extension_b4} and \ref{lemma:extension_b5}) as the above requires to strengthen the assumptions on the non-parametric function families. More specifically, the stronger assumption (b4) requires linear independence of the family $\mathcal{P}^{M}_{\B}$ to hold on every subset of a full measure set of $\R^{dM}$; whereas for $M=1$, only a non-zero measure set $\mathcal{Y}\subset\R^d$ suffices. In the experiments (Sec. \ref{sec:experiments_msm}, Figure \ref{fig:model_assumptions}), we illustrate this requirement by comparing estimation for non-zero and full measure sets. The above strategy ensures linear independence of product families when $T$ is a multiple of the lag $M$. In Theorem \ref{thm:linear_independence_joint_non_parametric} we use the induction technique for $T=a\cdot M$, with $a\in\mathbb{Z}^{+}$ to prove linear independence for any $T>M$.

\vspace{-.25em}
\subsection{Regime-dependent Causal Discovery}\label{sec:causal_discovery}

In regime-dependent time series \citep{saggioro2020reconstructing}, the SEM is time-dependent. However, it is \emph{causally stationary} \citep{assaad2022survey} when conditioned on discrete latent variables  $s_t\in\{1, \dots, K\}$ at each time step $t$. Therefore, the corresponding causal graph becomes a \emph{regime-dependent causal graph} which we denote as $\mathbf{G}_{1:K}:=\{\mathbf{G}_k\}_{k=1}^K \in (\times^K \mathbb{G})$. It can be defined as a collection of causal graphs such that for each $k\in\{1,\dots, K\}$, $\mathbf{G}_{k}\in\mathbb{G}$ encodes the SEM at time $t$ if $s_t=k$. To establish connections with the previous parametric high-order MSMs, we assume an \emph{additive noise model} (ANM; \citep{hoyer2008nonlinear}) with absence of instantaneous effects. Therefore, the corresponding regime-dependent SEM is expressed as follows
\vspace{-.5em}
\begin{multline}
    x^{(j)}_t = f^{(j)}_{s_t}(\{x^{(i)}_{t-\tau} | x^{(i)}_{t-\tau} \in \mathbf{Pa}^{(j)}_{s_t}(\tau), \\ 1\leq \tau \leq M \}) + \bm{\varepsilon}^{(j)}_{t}, \quad \bm{\varepsilon}^{(j)}_{t} \sim p_{\bm{\varepsilon}^{(j)}},
\end{multline}
\vspace{-1.75em}

where $\mathbf{Pa}^{(j)}_{s_t}(\tau)$ denotes the parents of variable $j$ with lag $\tau$, associated to $\mathbf{G}_k$ for $s_t=k$.
\vspace{-.5em}
\begin{definition}\label{def:regime_dependent_graph}
    We say that the regime-dependent graph of a MSM with order $M$ is identifiable up to permutations if for any $p,\tilde{p}\in \mathcal{M}^T(\Pi_{\A}^M, \mathcal{P}_{\B}^M)$, with respective regime-dependent graphs $\mathbf{G}_{1:K}\in (\times^{K} \mathbb{G})$ and $\tilde{\mathbf{G}}_{1:\tilde{K}}\in (\times^{\tilde{K}} \mathbb{G})$, such that $p(\x_{1:T})=\tilde{p}(\x_{1:T})$; $K=\tilde{K}$ and there exists a permutation $\sigma$ such that $\mathbf{G}(k) = \tilde{\mathbf{G}}(\sigma(k))$ for $k\in K$.
\end{definition}
\vspace{-.75em}
We establish identifiability for regime-dependent time series as follows.
\begin{corollary}\label{cor:causality_identifiability}
Consider the Markov switching model family with order $M$ defined in Equation (\ref{eq:msm_finite_mixture}) under non-linear Gaussian families $\mathcal{M}(\mathcal{I}_{\A}^M, \mathcal{G}_{\B}^M)$%, where $T = c\cdot M$, with $c,M\in\mathbb{Z}^+$
. Assume (i) from \ref{thm:identifiability_main} is satisfied;
\vspace{-.5em}
\begin{enumerate}[leftmargin=*]
    \item[(ii)] Independent noise: the transition distribution mean in  $\bm{m}(\cdot,b): \R^{dM} \rightarrow \R^d$ is analytic, and $\bm{\Sigma}(\cdot,b):=\bm{\Sigma}(b)$ is diagonal and constant w.r.t $(\x_{t-1},\dots,\x_{t-M})$;
    \item[(iii)] Minimality: if $x^{(i)}\notin \mathbf{Pa}^{(j)}_{k}(\tau), 1 \leq \tau \leq M$, then the $i$-th dimension of $\bm{m}(\x_{t-1},\dots,\x_{t-M},k)$ is constant w.r.t. $x^{(j)}_{t-\tau}$.
\end{enumerate}
\vspace{-.5em}
Then, the regime-dependent graph is identifiable up to permutations (Definition \ref{def:regime_dependent_graph}).
\end{corollary}
\vspace{-1em}
See Appendix \ref{app:causality_proof} for the proof. It suffices to show that the regime dependent structure is recovered up to permutations, given the identifiable transition derivatives from Theorem \ref{thm:identifiability_main}. Under assumptions (i-iii), we note that other structures in traditional causal time series analysis such as the \emph{full time graph} \citep{peters2017elements} are not identifiable in the context of high-order MSMs; as we have no access to the discrete latents $\s_{M:T}$, but rather their distribution up to permutations.
\begin{figure*}
    \centering
    \begin{subfigure}{.32\linewidth}
        \centering
        \includegraphics[width=\linewidth]{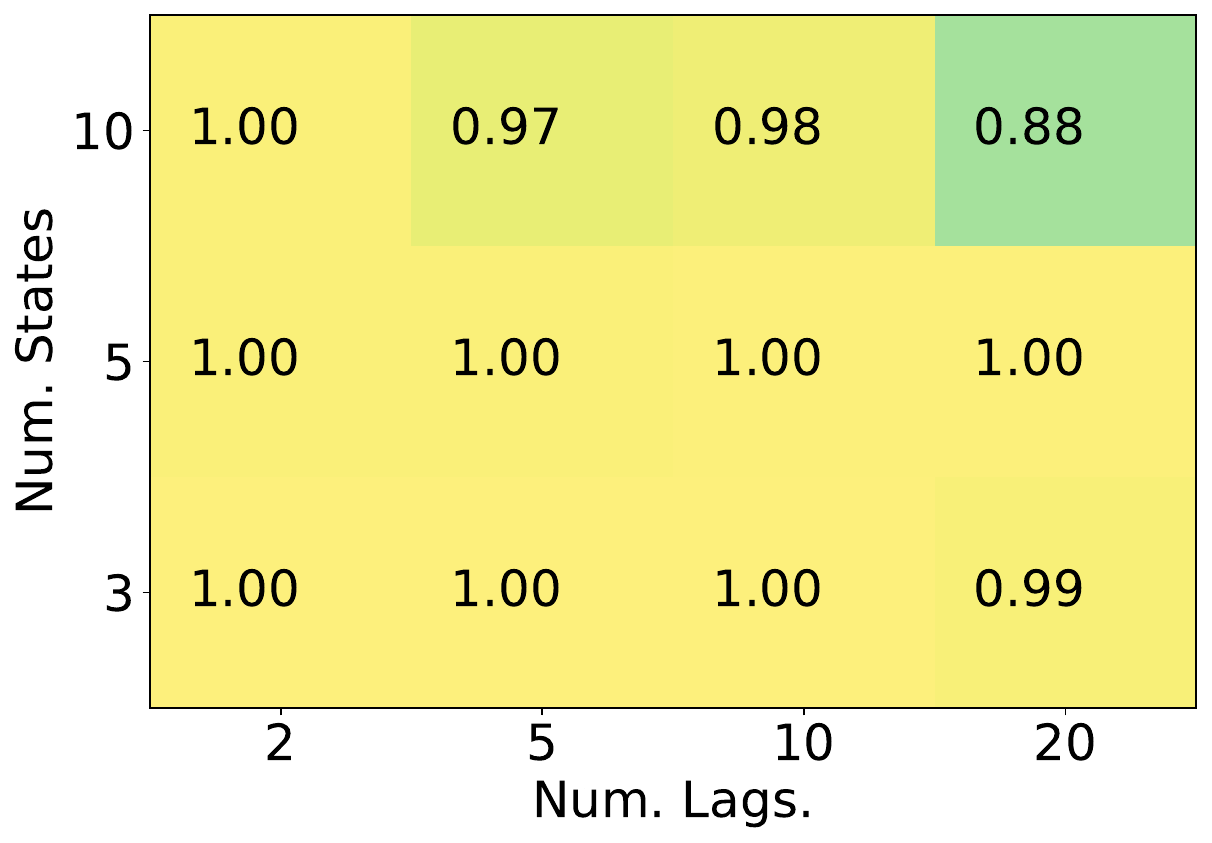}
        \vspace{-1.5em}
        \caption{Low Sparsity.}
        \label{fig:low_sparsity}
    \end{subfigure}
    \begin{subfigure}{.32\linewidth}
        \centering
        \includegraphics[width=\linewidth]{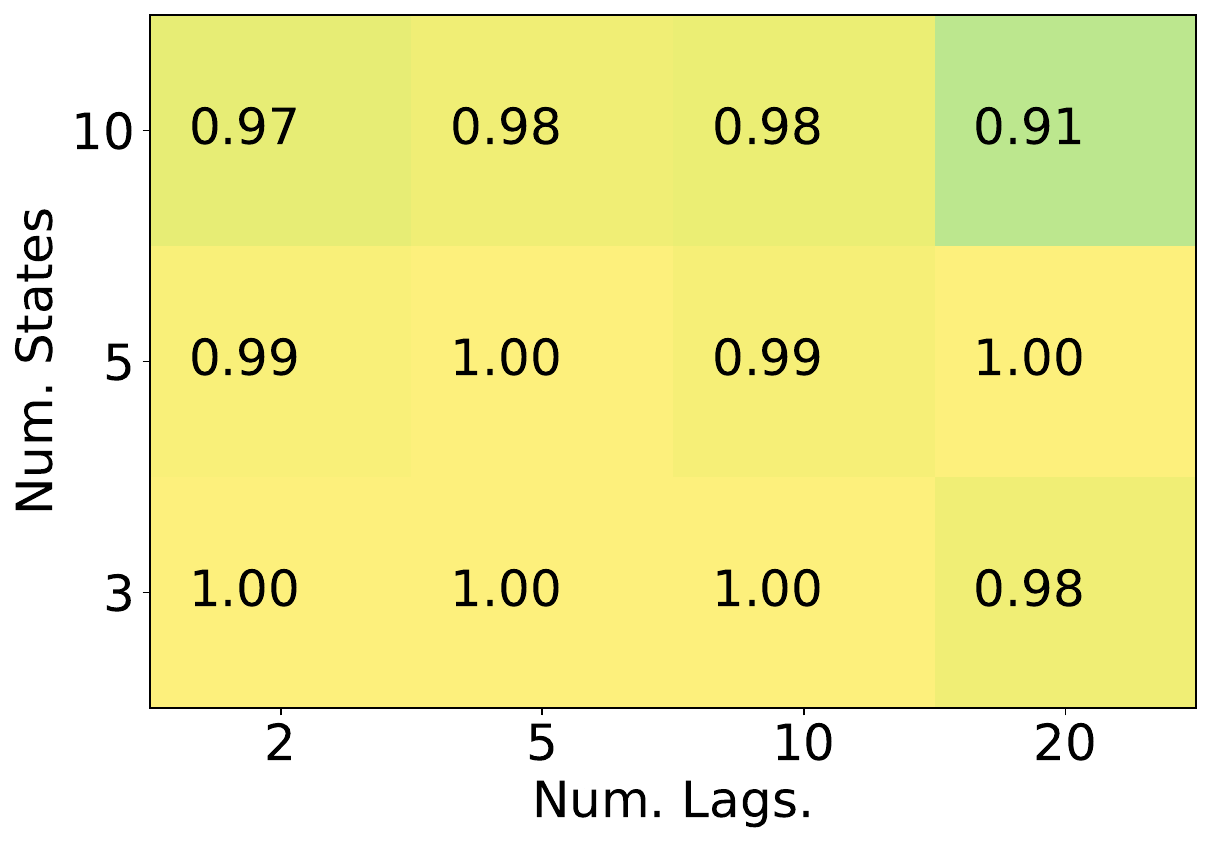}
        \vspace{-1.5em}
        \caption{High Sparsity.}
        \label{fig:high_sparsity}
    \end{subfigure}
    \begin{subfigure}{.32\linewidth}
        \centering
        \includegraphics[width=\linewidth]{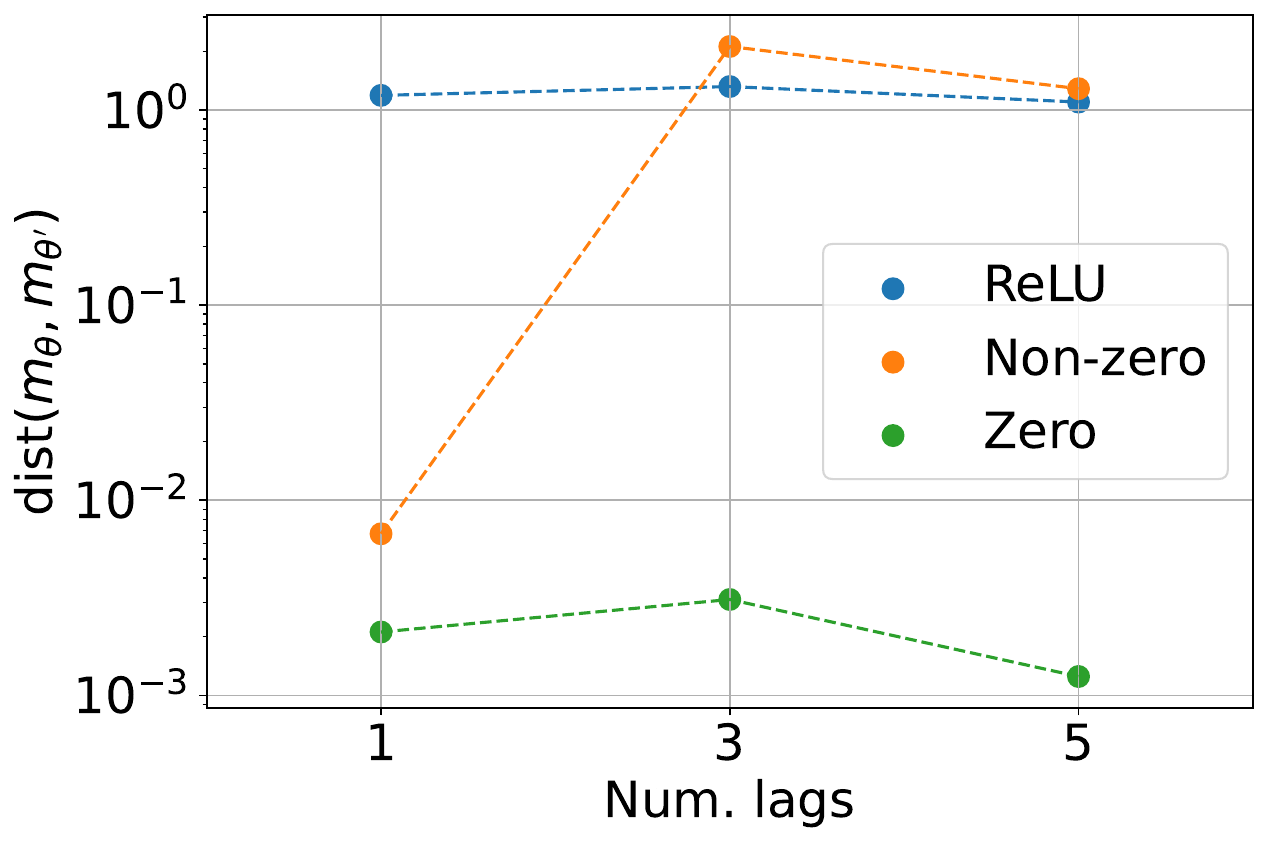}
        \vspace{-1.5em}
        \caption{Model assumptions.}
        \label{fig:model_assumptions}
    \end{subfigure}
    \vspace{-.5em}
    \caption{Synthetic experiments on high-order MSMs: averaged $F_1$ score on (a) low and (b) high sparsity settings; (c) averaged $L_2$ distance using different model assumptions.}
    \label{fig:experiments_synth}
    \vspace{-1em}
\end{figure*}

\vspace{-1.em}
\section{Estimation}
We use Expectation Maximisation (EM) for efficient mixture model estimation \citep{bishop2006pattern}, assuming we have a dataset of $N$ sequences of length $T$. Below provide update rules for a sample $\x_{1:T}$, but note that we use mini-batch stochastic gradient ascent when the number of sequences in the dataset is large. As mentioned, we assume a first-order Markov chain for the structure of the discrete states, where their posterior distribution at time $t$, $\{\gamma_{t,k}(\x_{1:T}) = p_{\mparam}(s_t=k|\x_{1:T}) \ \}$, can be computed using forward and backward messages: $\{ \alpha_{t,k}(\x_{1:t}) = p_{\mparam}(\x_{1:t},s_t=k) \}$, and $\{ \beta_{t,k}(\x_{t+1:T}) = p_{\mparam}(\x_{t+1:T}|s_t=k) \}$ respectively:
\begin{equation}
    \gamma_{t,k}(\x_{1:T}) = \frac{\alpha_{t,k}(\x_{1:t})\beta_{t,k}(\x_{t+1:T})}{\sum_{k=1}^K\alpha_{t,k}(\x_{1:t})\beta_{t,k}(\x_{t+1:T})},
\end{equation}
where $\alpha_{t,k}(\x_{1:t})$ and $\beta_{t,k}(\x_{1:t})$ are computed as follows:
\begin{multline}
\alpha_{t,k}(\x_{1:t}) = p_{\mparam}(\x_t | \x_{t-1:t-M}, s_t = k)  \\
\times \sum_{k'=1}^K p_{\mparam}(s_t = k | s_{t-1} = k') \alpha_{t-1,k'}(\x_{1:t-1}),
\end{multline}
\vspace{-1em}
\begin{multline}
   \beta_{t,k}(\x_{t+1:T}) = \sum_{k'=1}^K p_{\mparam}(\x_{t+1} | \x_{t:t+1-M}, s_{t+1} = k') \\
\times p_{\mparam}(s_{t+1} = k' | s_{t} = k) \beta_{t+1,k'}(\x_{t+1:T}). 
\end{multline}
Given that we parametrise the transition distributions using analytic neural networks, we adopt Generalised EM (GEM) \citep{dempster1977maximum}, where the following gradient ascent step is performed:
\begin{multline}\label{eq:update_rule}
    \mparam^{\text{new}} \leftarrow \mparam^{\text{old}} +
    \eta \sum_{t=M+1}^T \sum_{k=1}^K \gamma_{t,k}(\x_{1:T}) \\ \times \nabla_{\mparam} \log p_{\mparam}(\x_t| \x_{t-1:t-M}, s_t=k).
\end{multline}
The gradient term can be computed using back-propagation. This approach is well-established in the literature \citep{halva2020hidden}, and convergence to a local maximum of the likelihood is guaranteed under the large data limit.
\vspace{-.25em}
\section{Experiments}\label{sec:exp}
We first conduct experiments on synthetic data to (i) explore the scalability of our approach in learning regime-dependent graphs and (ii) empirically verify whether the assumptions in Theorem \ref{thm:identifiability_main} are necessary for identifiability. Furthermore, we motivate our approach using real brain activity data.

\vspace{-.5em}
\subsection{Synthetic Data}\label{sec:experiments_msm}

We generate $N=10000$ sequences of length $T=200$ and $d=5$ dimensions using MSMs described in Section \ref{sec:background}. The estimated functions that parameterise the transitions are evaluated using the permutation with the lowest error. We use diagonal, fixed covariance matrices and generate analytic transition means using random neural networks with cosine activations. Locally connected networks \citep{zheng2018dags} are used, with data dependencies following a sampled regime-dependent graph with a predefined sparsity ratio. See Appendix \ref{app:experiment_details} for details on data generation, model evaluation and training.

\paragraph{Regime-dependent causal structure estimation} To evaluate our approach, we generate data with different lags $M\in\{2, 5, 10, 20\}$, states $K\in\{3, 5, 10\}$, and sparsity settings. The estimated causal structure is computed via thresholding the Jacobian of the estimated transition functions. We consider moderate sparsity (up to $20$ parents per variable) and high sparsity ($5$ parents per variable). To maintain similar sparsity levels across different lags, the sparsity ratio increases with $M$. Results are reported respectively in Figures \ref{fig:low_sparsity} and \ref{fig:high_sparsity} respectively, where we compute the averaged $F_1$ score across states after accounting for the permutation. High $F_1$ scores are consistently observed for increasing states and lags, demonstrating the effectiveness of our identifiable MSM for regime-dependent causal discovery with high-order temporal dependencies.

\paragraph{Assumption violations} We have only shown sufficient conditions for identifiability under analytic transition functions (assumption (ii) in Theorem \ref{thm:identifiability_main}). %Given that our identifiability result does not correspond to a necessity statement, 
To empirically verify identifiability under weaker assumptions, we experiment with: ReLU networks (fig. \ref{fig:model_assumptions}, ReLU); piece-wise analytic functions with a non-zero measure intersection (fig. \ref{fig:model_assumptions}, Non-zero); and analytic functions assumed in (ii), which have zero-measure intersection (fig. \ref{fig:model_assumptions}, Zero). Although ReLU networks violate (ii), identifiability should hold if assumption (d3) in Theorem \ref{thm:linear_independence_nonlinear_gaussian} holds (zero-measure intersection of Gaussian moments). Therefore, we sample random ReLU networks with the same causal structure across regimes. 

\begin{figure*}
    \centering
    \begin{subfigure}{.32\linewidth}
        \centering
        \includegraphics[width=\linewidth]{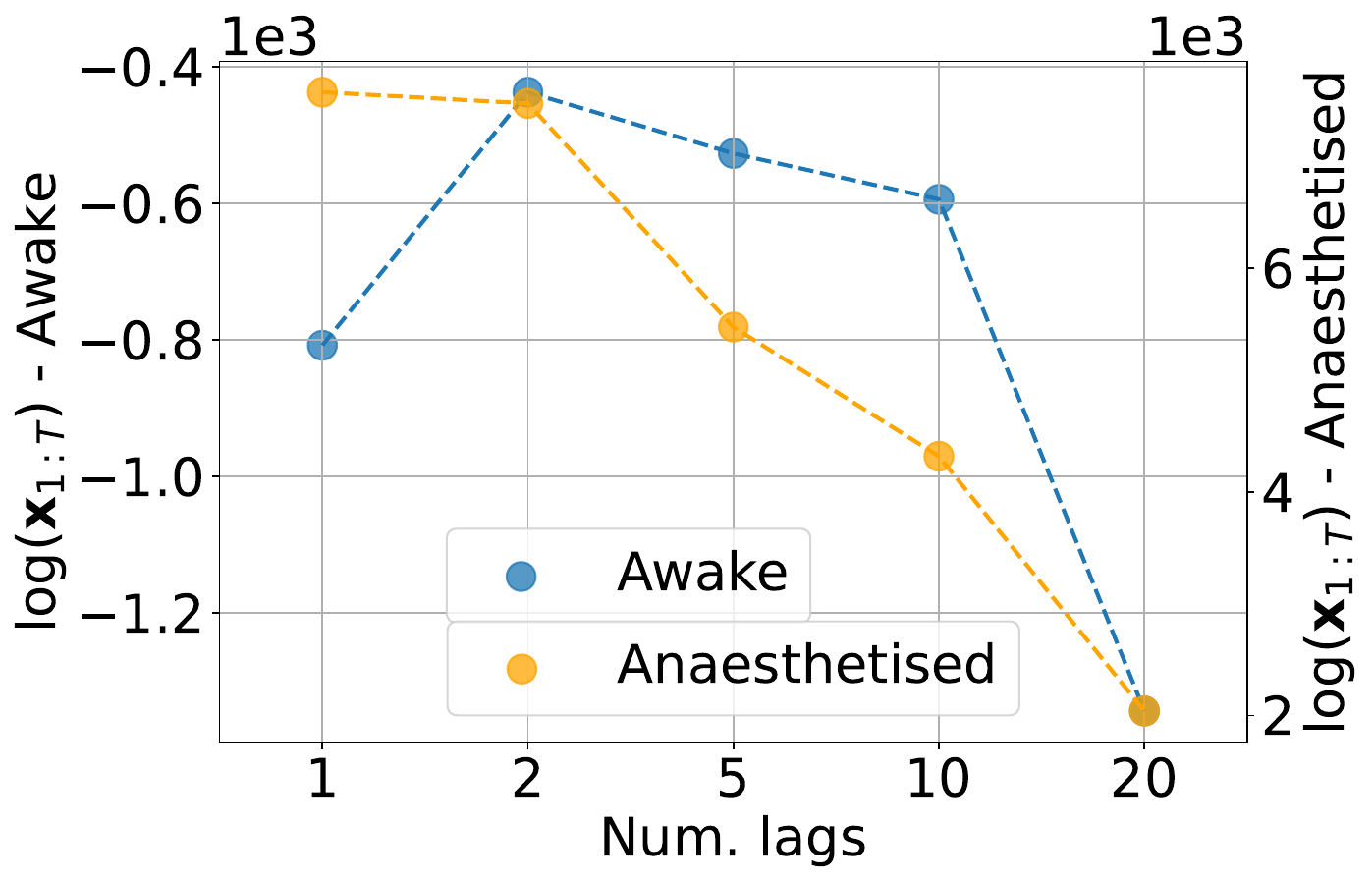}
        \vspace{-1.5em}
        \caption{}
        \label{fig:loglikelihood}
    \end{subfigure}
        \begin{subfigure}{.65\linewidth}
        \centering
        \includegraphics[width=\linewidth]{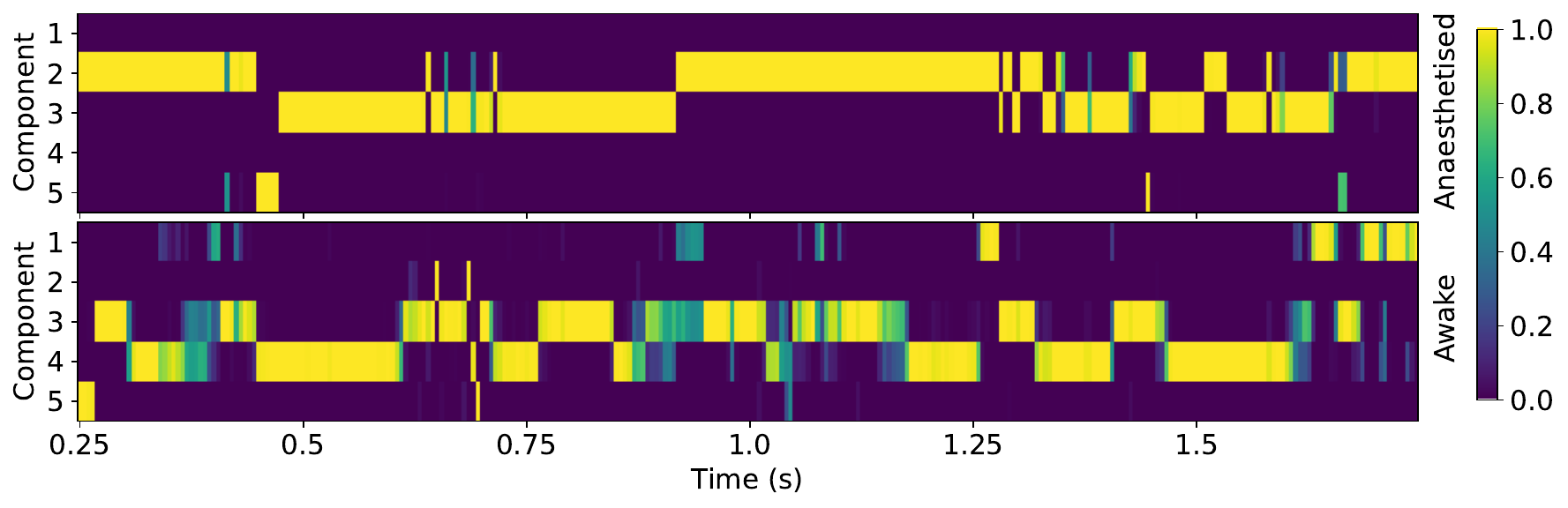}
        \vspace{-1.5em}
        \caption{}
        \label{fig:example_neuro}
    \end{subfigure}
    \vspace{-1em}
    \caption{(a) Test log-likelihood of ECoG data using different lags. (b) Posterior distribution of an ECoG epoch on Awake and Anaesthetised conditions using $K=5$ states and $M=2$ lags.}
    \label{fig:neuroscience_exp}
    \vspace{-1em}
\end{figure*}

Figure \ref{fig:model_assumptions} shows the results with averaged $L_2$ distances across regimes after accounting for permutation. ReLU networks are estimated with higher error as they directly violate assumption (ii) and the equal causal structures across regimes reduce the chances of meeting (d3) in Theorem \ref{thm:linear_independence_nonlinear_gaussian}. Regarding piece-wise analytic functions (Non-zero), we observe an abrupt increase in $L_2$ distance for higher lags ($M > 1$) when compared to the results with analytic functions. This confirms our theoretical findings discussed in Section \ref{sec:msm_theory}, where for $M>1$ the assumptions need to be strengthened to ensure zero measure intersection, whereas for $M=1$ a non-zero measure intersection is allowed.

\vspace{-.5em}
\subsection{Brain Activity Data}

To demonstrate the scalability of our method, we apply it to high-density electrocorticography (ECoG) brain activity from the NeuroTycho database\footnote{http://neurotycho.org/}, originally presented by \citet{yanagawa2013}. ECoG data consisting of signals from 128 electrodes across the brain was recorded from a macaque monkey under two conditions: normal wakefulness (Awake) and loss of consciousness induced through propofol anaesthesia (Anaesthetised). Our focus is to assess whether our identifiable high-order MSM captures different dynamics across conditions, enabling hypothesis testing in neuroscience. Based on prior work \citep{mediano2023spectrally}, we expect to observe more rapidly changing dynamics in the awake condition, captured by more frequent transitions between the states in our MSM. 

The ECoG recording is sampled at 1kHz and comprises ``awake'' and ``anaesthetised'' segments, each lasting 929.8 and 650.65 seconds respectively. We select 21 electrodes approximately corresponding to the visual cortex of the brain. We first apply a second-order Butterworth notch filter at 50Hz to eliminate line noise. Then, we downsample the data to 200Hz, standardise each channel independently, and chunk each sequence into epochs of 2 seconds ($T=400$). This yields $464$ ``awake'' and $325$ ``anaesthetised'' segments, with $50$ of each saved for testing.

Figure \ref{fig:loglikelihood} shows test log-likelihoods for MSMs with different lags. The ``awake'' condition log-likelihoods are lower, which illustrate the increased complexity in dynamics. Our identifiable MSM with $M=1$ performs best on the ``anaesthetised'' condition, while $M=2$ is better on the ``awake'' condition. This suggests $M=1$ is better for simpler dynamics, whereas higher-orders are better for more complex ones. For $M > 1$, the log-likelihood decreases with increasing lags, possibly due to the increased training complexity. Table \ref{tab:table_neuro} presents state transition frequencies (in Hz), with consistently higher values in the ``awake'' condition. To suppress transition artifacts, we convolve the posterior signal using a uniform kernel of length $3$. Furthermore, the example sequence in Figure \ref{fig:example_neuro} shows prolonged state preservation in the ``anaesthetised'' condition. This result motivates identifiable MSMs for neuroscience, complementing existing methodologies \citep{mediano2023spectrally}.

\begin{table}[]
    \centering
    \caption{State transition frequency (in Hz) on visual cortex ECoG data with $M=2$ using different states.}
    \begin{tabular}{l|c|c|c}
        \toprule
        \multirow{2}{*}{Condition} & \multicolumn{3}{c}{K} \\
        \cline{2-4}
        & 3 & 5 & 10\\
        \midrule
        Awake & 17.60 & 19.62 & 23.74\\
        Anaesthetised & 5.75 & 6.51 & 14.64 \\
        \bottomrule
    \end{tabular}
    \label{tab:table_neuro}
    \vspace{-.75em}
\end{table}
\vspace{-.25em}
\section{Conclusions}

In this work, we prove identifiability of regime-dependent causal discovery using identifiable high-order Markov Switching Models (MSMs). Our key contribution is the generalisation of identifiable first-order MSMs to higher-order temporal dependencies via strengthening the assumptions in the non-parametric case. This enables model parametrisations via analytic Gaussian transitions. We verify our theoretical findings empirically through synthetic experiments and demonstrate the applicability of our approach to realistic domains, such as hypothesis testing in neuroscience. Future studies could be focused on incorporating instantaneous effects, including full time graph identifiability via consistent state identification, or leveraging identifiable high-order MSMs for nonstationary latent causal models.

%\begin{contributions} % will be removed in pdf for initial submission 
					  % (without ‘accepted’ option in \documentclass)
                      % so you can already fill it to test with the
                      % ‘accepted’ class option
%    Briefly list author contributions. 
%    This is a nice way of making clear who did what and to give proper credit.
%    This section is optional.

%    H.~Q.~Bovik conceived the idea and wrote the paper.
%    Coauthor One created the code.
%    Coauthor Two created the figures.
%\end{contributions}

%\begin{acknowledgements} % will be removed in pdf for initial submission,
						 % (without ‘accepted’ option in \documentclass)
                         % so you can already fill it to test with the
                         % ‘accepted’ class option
%    Briefly acknowledge people and organizations here.

%    \emph{All} acknowledgements go in this section.
%\end{acknowledgements}

% References
\bibliographystyle{plainnat}
\renewcommand{\bibsection}{\subsubsection*{References}}

\bibliography{ci4ts2024-template}

\newpage

\onecolumn

\title{Identifying Nonstationary Causal Structures with \\  High-Order Markov Switching Models (Supplementary Material)}
\maketitle

\appendix
\section{Identifiability in finite mixture models}\label{app:non_parametric_mixtures}

Our theoretical framework uses finite mixture model results from \citet{yakowitz1968identifiability}, which show identifiability of finite mixtures via linear independence of the family of mixing components. Consider a distribution family with functions defined on $\x\in\R^d$,
\begin{equation}\label{eq:family_functions}
    \mathcal{F}_{\A} := \{F_{a}(\x) | a \in \A \},
\end{equation}
where $F_{a}(\x)$ is a $d$-dimensional CDF. The index set $\A$ is assumed to satisfy that $F_{a}(\x)$, as a function of $(\x,a)$, is measurable on $\R^d\times\A$. We introduce the notion of linear independence under finite mixtures of a family $\mathcal{F}_{\A}$.
\begin{definition}
\label{def:linear_independence_finite_mixture}
A family of functions $\mathcal{F}_{\A}$ (Eq. (\ref{eq:family_functions})) is said to contain linearly independent functions under finite mixtures, if for any $\A_0 \subset \A$ such that $|\A_0| < +\infty$, the functions in $\{f_a(\x) | a \in \A_0 \}$ are linearly independent.
\end{definition}
The above definition is a weaker requirement of linear independence on function families as it allows linear dependence from the linear combination of infinitely many other functions. Consider the following finite mixture distribution family:
\begin{equation}
    \mathcal{H}_{\A} := \{ H(\x) = \sum_{i=1}^L c_i F_{a_i}(\x) | L < +\infty, a_i \in \A, a_i \neq a_j, \forall i \neq j, \sum_{i=1}^L c_i = 1 \},
\end{equation}
which is defined from a linear combination of CFDs in $\mathcal{F}_{\A}$. Now we specify the definition of \emph{identifiable finite mixture family} following \citet{yakowitz1968identifiability}.
\begin{definition}
\label{def:identifiability_finite_mixture}
The finite mixture family $\mathcal{H}$ is said to be identifiable up to permutations, when for any two finite mixtures $H(x) = \sum_{i=1}^L c_i F_{a_i}(\x)$ and $\tilde{H}(x) = \sum_{i=1}^{\tilde{L}} \tilde{c}_i F_{\tilde{a}_i}(\x)$, $H(\x) = \tilde{H}(\x)$ for all $\x \in \mathbb{R}^d$, if and only if $L = \tilde{L}$ and for each $1 \leq i \leq L$ there is some $1 \leq j \leq \tilde{L}$ such that $c_i = \tilde{c}_j$ and $F_{a_i}(\x) = F_{\tilde{a}_j}(\x)$ for all $\x \in \mathbb{R}^d$.
\end{definition}
Then, identifiability of finite mixture models is stated as follows.
\begin{proposition}\citep{yakowitz1968identifiability}
\label{prop:mixture_cdf_identifiability}
The finite mixture distribution family $\mathcal{H}$ is identifiable up to permutations, if and only if functions in $\mathcal{F}$ are linearly independent under the finite mixtures. 
\end{proposition}

\section{Proof of Theorem \ref{thm:identifiability_main}}\label{app:proof_identifiability}
\paragraph{Sketch of the proof:} We organise the proof strategy into 4 steps.
\begin{enumerate}
    \item We show the requirement for identifiability is linear independence of the joint distribution family.
    \item We provide linear independence results for products of non-parametric functions. First, we state the linear dependence result from \citet{balsells-rodas2024on}, and we then strengthen the assumptions to allow products of $M$ functions.
    \item We prove linear independence of the joint distribution family for non-parametric transitions of order $M$.
    \item We show the parametric assumptions (i), (ii) satisfy linear independence of the joint distribution family.
\end{enumerate}
We note the strategy follows closely from \citet{balsells-rodas2024on}, where identifiability is shown for first-order autoregressive dependencies. 
\subsection{Linear independece renders identifiability}

\citet{balsells-rodas2024on} shows that Proposition \ref{prop:mixture_cdf_identifiability} can be generalised to CFDs defined on $\x_{1:T}\in\R^{Td}$. We note the joint distribution family of the MSM with order $M$ $\mathcal{P}_{\A,\B}^{M,T}:= \Pi_{\A}^M\otimes(\otimes_{t=M+1}^T \mathcal{P}_{\B}^M)$ has linear independent components  if and only if its CDF also contains linear independent components. We include the following extension to Proposition \ref{prop:mixture_cdf_identifiability} which is adapted from \citet{balsells-rodas2024on}.

\begin{proposition}
\label{prop:mixture_pdf_identifiability}
\textbf{(Adapted from \citep{balsells-rodas2024on})} Consider the distribution family given by Eq. \ref{eq:msm_finite_mixture}.
Then the joint distribution in $\mathcal{M}^T(\Pi_{\A}^M,\mathcal{P}_{\B}^M)$ is identifiable up to permutations if and only if functions in $\mathcal{P}^{M,T}_{\A,\B}$ are linearly independent under finite mixtures.
\end{proposition}

We also adapt the following result from \citep{balsells-rodas2024on} and provide the proof for completeness. 
\begin{theorem}\label{thm:identifiability_msm} \textbf{(Adapted from \citep{balsells-rodas2024on})}
    Assume the functions in $\mathcal{P}^{T,M}_{\A,\B}$ are linearly independent under finite mixtures, then the distribution family $\mathcal{M}^T(\Pi_{\A}^M,\mathcal{P}_{\B}^M)$ is identifiable as defined in Def \ref{def:identifiability}.
\end{theorem}
\begin{proof}
Assuming linear independence under finite mixtures of $\jointfamily$, implies identifiability up to permutation as defined in Definition \ref{def:identifiability_finite_mixture} (finite mixture model case). Then, for $p_1(\x_{1:T})$ and $p_1(\x_{1:T})$ from Definition \ref{def:identifiability}, we have $C=\tilde{C}$, and for every $1\leq i \leq C$, there exists $1\leq j \leq \tilde{C}$ such that $c_i=\tilde{c}_j$ and:
\begin{equation}
p_{a^i}(\x_{1:M})\prod_{t=M+1}^T p_{b_t^i}(\x_t | \x_{t-1}, \dots, \x_{t-M}) =  p_{\tilde{a}^j}(\x_{1:M})\prod_{t=M}^T p_{\tilde{b}_t^j}(\x_t | \x_{t-1}, \dots, \x_{t-M}), \quad \forall \x_{1:T} \in \mathbb{R}^{Td}.    
\end{equation}
Given that we have conditional PDFs, if the joint distributions are equal on $\x_{1:T}$, then the distributions on $\x_{1:T-1}$ are also equal:
\begin{equation}
p_{a^i}(\x_{1:M})\prod_{t=M+1}^{T-1} p_{b_t^i}(\x_t | \x_{t-1}, \dots, \x_{t-M}) =  p_{\tilde{a}^j}(\x_{1:M})\prod_{t=M}^{T-1} p_{\tilde{b}_t^j}(\x_t | \x_{t-1}, \dots, \x_{t-M}), \quad \forall \x_{1:T-1} \in \mathbb{R}^{(T-1)d}.    
\end{equation}
Therefore, we have $p_{b_T^i}(\x_T | \x_{T-1}, \dots, \x_{T-M}) = p_{\tilde{b}_T^j}(\x_T | \x_{T-1}, \dots, \x_{T-M})$ for all $\x_t, \dots, \x_{t-M}\in \R^d$.  We can follow the same reasoning for other time indices to have $p_{b_t^i}(\x_t | \x_{t-1}, \dots, \x_{t-M}) = p_{\tilde{b}_t^j}(\x_t | \x_{t-1}, \dots, \x_{t-M})$ for all $t > M, \x_t, \dots, \x_{t-M}\in \R^d$. Similar logic applies to the initial distribution, where we have $p_{a^i}(\x_{1:M})=p_{\tilde{a}^j}(\x_{1:M})$ for all $\x_1,\dots,\x_M\in\R^d$; and from Proposition \ref{prop:mixture_pdf_identifiability}, we have $K_0=\tilde{K}_0$. Given $C=\tilde{C}$ we also have $K=\tilde{K}$.

Finally, if there exists $t_1 \neq t_2$ such that $b_{t_1}^i = b_{t_2}^i$ but $\tilde{b}_{t_1}^j \neq \tilde{b}_{t_2}^j$, we have for any $\bm{\alpha}\in \mathbb{R}^{Md},  \bm{\beta} \in\R^d$:
\begin{align*}
p_{\tilde{b}_{t_1}^j}(\x_{t_1} = \bm{\beta} | \x_{t_1 -1:t_1 - M} = \bm{\alpha}) &= p_{b_{t_1}^i}(\x_{t_1} = \bm{\beta} | \x_{t_1 -1:t_1 - M} = \bm{\alpha}) \\
&= p_{b_{t_2}^i}(\x_{t_2} = \bm{\beta} | \x_{t_2 -1:t_2-M} = \bm{\alpha}) \\
&= p_{\tilde{b}_{t_2}^j}(\x_{t_2} = \bm{\beta} | \x_{t_2 -1:t_2-M} = \bm{\alpha}),
\end{align*}
which implies linear dependence of $\mathcal{P}^M_{\B}$. We note this contradicts the assumption of linear independence of the joint distribution $\jointfamily$, as we should have
\begin{equation}
\sum_{i}\sum_{j} \gamma_{ij} p_{b^i_{1:T-1}}(\x_{1:T-1}) p_{b^j_T}(\x_T | \x_{T-1:T-M}) = 0, \quad \forall \x_{1:T} \in \mathbb{R}^{(T-1)d} \times \mathbb{R}^{d},
\end{equation}
with $\gamma_{ij}=0,  1 \leq i \leq C/K, 1 \leq j \leq K$. However, given linear dependence on $\mathcal{P}^M_{\B}$, we can the rearange the terms and set $\gamma_{i,j} = \gamma_j \neq 0, 1\leq j \leq K$ such that the above equation is satisfied.

\end{proof}

The next step is to show conditions under which the joint distribution family $\jointfamily$ is linearly independent under finite mixtures.

\subsection{Linear independence of non-parametric product families}

The strategy to prove linear independence on the non-parametric joint distribution family is to show linear independence for consecutive products of distributions. As an example, for $M=1$, we have $\{p(\x_t|\x_{t-1},s_t)p(\x_{t+1}|\x_{t},s_{t+1})\}$ with one overlapping variable. Increasing $M$ increases the number of overlaps: $\{\prod_{m=1}^M p(\x_{t+m} | \x_{t-M+m:t-1+m}, s_{t+m})\}$. Notably for $M>1$, we observe two types of overlapping variables when computing the above joint probability of $M$ consecutive observations:
\begin{enumerate}
    \item Observed-conditioned overlaps: e.g. $\x_{t+M}, \dots, \x_{t+1}$ in the example, or the observed variables in the above joint probability product.
    \item Conditioned-conditioned overlaps: e.g. $\x_{t-1}, \dots, \x_{t-M}$ in the example, or the conditioned variables in the above joint probability product.
\end{enumerate}
Therefore, the increase of overlapping variables complicates the verification of linear independence of the joint distribution family.

\subsubsection{Preliminaries}

We start the following result from \citet{balsells-rodas2024on}, which shows the conditions under which linear independence can be preserved for product functions of consecutive variables with $M=1$ overlap ($\y$).

\begin{lemma}
\label{lemma:linear_independence_two_nonlinear_gaussians}
\citep{balsells-rodas2024on} Consider two families $\mathcal{U}_I := \{u_i(\y, \x) | i \in I \}$ and $\mathcal{V}_J := \{v_j(\z, \y) | j \in J \}$ with $\x \in \mathcal{X}, \y \in \mathbb{R}^{d_y}$ and $\z \in \mathbb{R}^{d_z}$. We further assume the following assumptions:
\begin{itemize}
    \item[(a1)] Positive function values: $u_i(\y, \x) > 0$ for all $i \in I, (\y, \x) \in \mathbb{R}^{d_y} \times \mathcal{X}$. Similar positive function values assumption applies to $\mathcal{V}_J$: $v_j(\z, \y) > 0$ for all $j \in J, (\z, \y) \in \mathbb{R}^{d_z} \times \mathbb{R}^{d_y}$.
    \item[(a2)] Unique indexing: for $\mathcal{U}_I$, $i \neq i' \in I \Leftrightarrow \exists \ \x, \y \text{ s.t. } u_i(\x, \y) \neq u_{i'}(\x, \y)$. Similar unique indexing assumption applies to $\mathcal{V}_J$;
    \item[(a3)] Linear independence \textcolor{black}{under finite mixtures} on specific non-zero measure subsets for $\mathcal{U}_I$: for any non-zero measure subset $\mathcal{Y} \subset \mathbb{R}^{d_y}$, $\mathcal{U}_I$ contains linearly independent functions \textcolor{black}{under finite mixtures} on $(\y, \x) \in \mathcal{Y} \times \mathcal{X}$. 
    \item[(a4)] Linear independence \textcolor{black}{under finite mixtures} on specific non-zero measure subsets for $\mathcal{V}_J$: there exists a non-zero measure subset $\mathcal{Y} \subset \mathbb{R}^{d_y}$, such that for any non-zero measure subsets $\mathcal{Y}' \subset \mathcal{Y}$ and $\mathcal{Z} \subset \mathbb{R}^{d_z}$, $\mathcal{V}_J$ contains linearly independent functions \textcolor{black}{under finite mixtures} on $(\z, \y) \in \mathcal{Z} \times \mathcal{Y}'$;
    \item[(a5)] Linear dependence \textcolor{black}{under finite mixtures} for subsets of functions in $\mathcal{V}_J$ implies repeating functions: for any $\bm{\beta} \in \mathbb{R}^{d_y}$, any non-zero measure subset $\mathcal{Z} \subset \mathbb{R}^{d_z}$ and any subset $J_0 \subset J$ \textcolor{black}{such that $|J_0| < +\infty$}, $\{v_j(\z, \y = \bm{\beta}) | j \in J_0 \}$ contains linearly dependent functions on $\z \in \mathcal{Z}$ only if $ \exists \ j \neq j' \in J_0$ such that $v_j(\z, \bm{\beta}) = v_{j'}(\z, \bm{\beta})$ for all $\z \in \textcolor{black}{\mathbb{R}^{d_z}}$.
    \item[(a6)] Continuity for $\mathcal{V}_J$: for any $j \in J$, $v_j(\z, \y)$ is continuous in $\y \in \mathbb{R}^{d_y}$.
\end{itemize}
Then for any non-zero measure subset $\mathcal{Z} \subset \mathbb{R}^{d_z}$, $\mathcal{U}_I \otimes \mathcal{V}_J := \{v_j(\z, \y) u_i(\y, \x) | i \in I, j \in J \}$ 
contains linear indepedent functions under finite mixtures defined on $(\x, \y, \z) \in \mathcal{X} \times \mathbb{R}^{d_y} \times \mathcal{Z}$.
\end{lemma}

In summary, the result verifies that under (a1-a4), linear dependence could occur for every value of the overlapping variable ($\y$). However, this is not possible thanks to (a5-a6).

\citet{balsells-rodas2024on} provide a "proof by induction" technique using Lemma \ref{lemma:linear_independence_two_nonlinear_gaussians}. Similar strategies are not applicable for $M>1$. Given some initial and transition distributions defined in ... respectively, the base case $T=M+1$ can be proven using \ref{lemma:linear_independence_two_nonlinear_gaussians}. However, for $T = \tau$, $\tau > M+1$, we cannot directly use the results from Lemma \ref{lemma:linear_independence_two_nonlinear_gaussians} as the induction hypothesis does not satisfy (a3), due to $\y$ and $\z$ having different sizes. A similar "proof by induction" technique can be used when $\y$ and $\z$ are forced to have the same size. Given the increased size of the conditioned variables, this requires the induction technique to verify linear independence by grouping the distributions with $M$ consecutive products, $\{\prod_{m=1}^M p(\z_{t+m} | \z_{t-M+m:t-1+m}, s_{t+m})\}$. Consequently, the resulting family must satisfy (a1), (a2), and (a4-a6) to use Lemma \ref{lemma:linear_independence_two_nonlinear_gaussians}.

\subsubsection{Extending assumptions on product families}\label{app:extending_assumptions}

Below we explore assumptions on non-parametric families with $M+1$ variables (aligned with some transition distribution $p(\z_t|\z_{t-1:t-M}, s_t)$), such that the product of $M$ consecutives variables satisfies (a1), (a2), and (a4-a6). We note assumption (a2) is redundant, as if it holds for $U_I$ and $V_J$, then (a3) and (a4) hold respectively. However, the converse is not true. Therefore, we remove (a2) from our theoretical analysis for simplicity. Contrary to (a1) and (a6), we note that the extension of (a4-a5) to $M$ consecutive products requires them to be strengthened. Given a family $\mathcal{V}_J=\{v_j(\z,\y_M,\dots,\y_1)| j\in J\}$ with variables $(\z,\y_M,\dots,\y_1)$ defined on $\times_{m=1}^{M+1} \R^{d}$, we provide the assumption modifications following the enumeration presented in Lemma \ref{lemma:linear_independence_two_nonlinear_gaussians}.
\begin{enumerate}[leftmargin=2cm]
\item[(a1) $\rightarrow$ (b1)] Positive function values: $v_j(\z,\y_M,\dots,\y_1) > 0$ for all $j \in J, (\z,\y_M,\dots,\y_1) \in (\times_{m=1}^{M+1} \R^{d})$.
\item[(a4) $\rightarrow$ (b4)] There exist a non-zero measure set $\mathcal{Y}\subset (\times_{m=1}^M \R^d)$ with $\mu(\mathcal{Y}) = \mu(\times_{m=1}^M \R^d)$ and such that for every  non-zero measure sets $\mathcal{Y}' \subset \mathcal{Y}$, and $\mathcal{Z} \subset \R^{d}$, $\mathcal{V}_J$ contains linearly independent functions \textcolor{black}{under finite mixtures} on $(\z,\y_M,\dots,\y_1) \in \mathcal{Z} \times \mathcal{Y}'$;
\item[(a5) $\rightarrow$ (b5)] For any $(\bm{\beta}_m,\dots,\bm{\beta}_1) \in ( \times^m \R^{d})$, with $1\leq m \leq M$, any non-zero measure subsets $\mathcal{Z} \subset \mathbb{R}^{d}$, $\mathcal{Y}\subset (\times^{M-m}\R^{d})$, and any subset $J_0 \subset J$ \textcolor{black}{such that $|J_0| < +\infty$}, $\{v_j(\z,\y_M,\dots, \y_{m+1},\y_m=\bm{\beta}_m,\dots, \y_1=\bm{\beta}_1) | j \in J_0 \}$ contains linearly dependent functions on $(\z, \y_M,\dots,\y_{m+1}) \in \mathcal{Z}\times\mathcal{Y}$ only if $ \exists \ j \neq j' \in J_0$ such that $v_j(\z,\y_M,\dots, \y_{m+1},\y_m=\bm{\beta}_m,\dots, \y_1=\bm{\beta}_1) = v_j'(\z,\y_M,\dots, \y_{m+1},\y_m=\bm{\beta}_m,\dots, \y_1=\bm{\beta}_1)$ for all $(\z, \y_M,\dots,\y_{m+1}) \in \times^{M-m+1}\mathbb{R}^{d}$.
\item[(a6) $\rightarrow$ (b6)] Continuity for $\mathcal{V}_J$: for any $j \in J$, $v_j(\z,\y_M,\dots, \y_1)$ is continuous in $(\y_M,\dots, \y_1) \in ( \times_{m=1}^M \R^{d})$.
\end{enumerate}

We provide the following definitions which will be used in the results below.

\begin{definition}\label{def:set-projection}
    Let $\mathcal{A}$ be a set with $\mathcal{A}\subseteq(\times_{m=1}^n \R^d)$, and $S=\{s_1, \dots,  s_k\}$, where $1\leq k \leq n$ and $s_i\neq s_j$, $s_i,s_j\in\{1, \dots, n\}$. We define $\pi(\mathcal{A}, S)$ the projection of $\mathcal{A}$ into the dimensions of $S$ as
    $$\pi(\mathcal{A}, S) = \{(\x_{s_1}, \dots, \x_{s_k})\in(\times_{m=1}^k \R^d) : (\x_1, \dots, \x_n)\in\mathcal{A} \} \subseteq (\times_{m=1}^k \R^d)$$
\end{definition}

\begin{definition}\label{def:reverse-projection}
    Let $\mathcal{A}$ be a set with $\mathcal{A}\subseteq(\times_{m=1}^n \R^d)$. We define the reverse projection  given the last $n-k$ components of $\mathcal{A}$ : $(\x_n, \dots, \x_k)\in (\times_{m=k}^n \R^d)$, with $1\leq k \leq n$, as
    $$\mathcal{A}_{(\x_n, \dots, \x_k)} = \{(\x_{1}, \dots, \x_{k-1})\in(\times_{m=1}^{k-1} \R^d): (\x_1, \dots, \x_n)\in\mathcal{A} \}\subseteq (\times_{m=1}^{k-1} \R^d)$$
\end{definition}

Assumption (b4) extends to consecutive product functions as follows.

\begin{lemma}\label{lemma:extension_b4}
\textbf{B4 Extension.} Assume a family $\mathcal{V}_J=\{v_j(\z,\y_M,\dots,\y_1)| j\in J\}$ with variables $\z,\y_M,\dots,\y_1$ defined on $\times_{m=1}^{M+1} \R^{d}$, such that it satisfies (b1), (b4), (b5), and (b6).
%\begin{enumerate}
%\item[(b1)] Positive function values: $v_j(\z,\y_M,\dots,\y_1) > 0$ for all $j \in J, (\z,\y_M,\dots,\y_1) \in (\times_{m=1}^{M+1} \R^{d})$.
%\item[(b4)] There exist a non-zero measure set $\mathcal{Y}\subset (\times_{m=1}^M \R^d)$ with $\mu(\mathcal{Y}) = \mu(\times_{m=1}^M \R^d)$ such that for every non-zero measure sets $\mathcal{Y}' \subset \mathcal{Y}$, and $\mathcal{Z} \subset \R^{d}$, $\mathcal{V}_J$ contains linearly independent functions \textcolor{black}{under finite mixtures} on $(\z,\y_M,\dots,\y_1) \in \mathcal{Z} \times \mathcal{Y}'$;
%\item[(b5)] For any $(\bm{\beta}_m,\dots,\bm{\beta}_1) \in ( \times^m \R^{d})$, with $1\leq m \leq M$, any non-zero measure subsets $\mathcal{Z} \subset \mathbb{R}^{d}$, $\mathcal{Y}\subset (\times^{M-m}\R^{d})$, and any subset $J_0 \subset J$ \textcolor{black}{such that $|J_0| < +\infty$}, $\{v_j(\z,\y_M,\dots, \y_{m+1},\y_m=\bm{\beta}_m,\dots, \y_1=\bm{\beta}_1) | j \in J_0 \}$ contains linearly dependent functions on $(\z, \y_M,\dots,\y_{m+1}) \in \mathcal{Z}\times\mathcal{Y}$ only if $ \exists \ j \neq j' \in J_0$ such that $v_j(\z,\y_M,\dots, \y_{m+1},\y_m=\bm{\beta}_m,\dots, \y_1=\bm{\beta}_1) = v_j'(\z,\y_M,\dots, \y_{m+1},\y_m=\bm{\beta}_m,\dots, \y_1=\bm{\beta}_1)$ for all $(\z, \y_M,\dots,\y_{m+1}) \in \times^{M-m+1}\mathbb{R}^{d}$.
%\item[(b6)] Continuity for $\mathcal{V}_J$: for any $j \in J$, $v_j(\z,\y_M,\dots, \y_1)$ is continuous in $(\y_M,\dots, \y_1) \in ( \times_{m=1}^M \R^{d})$.
%\end{enumerate}

Then, for any $n$ with $2\leq n \leq M$, there exist a non-zero measure set $\mathcal{Y}\subset (\times_{m=1}^M \R^d)$ with $\mu(\mathcal{Y}) = \mu(\times_{m=1}^M \R^d)$ such that for every non-zero measure sets $\mathcal{Y}' \subset \mathcal{Y}$, and $\mathcal{Z} \subset (\times_{m=1}^n \R^{d})$, the family 
$$(\otimes_{m=1}^n \mathcal{V}_J) := \{v_{j_n}(\z_n,\dots,\z_1,\y_M,\dots, \y_{n})\dots v_{j_1}(\z_1,\y_M,\dots,\y_1)|(j_n,\dots,j_1)\in (\times_{m=1}^n J)\},$$ 
contains linear independent functions under finite mixtures in $(\z_n,\dots,\z_1,\y_M,\dots,\y_1)\in \mathcal{Z} \times \mathcal{Y}'$.
\end{lemma}
\begin{proof}
    We prove the above statement by induction, where we start with the base case ($n=2$) as follows.

    \paragraph{Case $n=2$.} 

    Assume the statement is false. Then, for every set $\mathcal{Y}\subset (\times_{m=1}^M \R^d)$ with $\mu(\mathcal{Y}) = \mu(\times_{m=1}^M \R^d)$, there exists non-zero measure sets $\mathcal{Z} \subset (\R^{d} \times \R^{d})$ and $\mathcal{Y}'\subset \mathcal{Y}$ such that $\mathcal{V}_J \otimes \mathcal{V}_J$ contains linear dependent functions. This means the following linear dependence condition must be satisfied
    \begin{multline}
        \sum_{(j_1,j_2)\in S_0} \gamma_{j_1 j_2} v_{j_2}(\z_2 , \z_1 , \y_M, \dots, \y_2 ) v_{j_1}(\z_1 , \y_M, \dots , \y_1 ) = 0, \\ \forall (\z_2 ,\z_1 ,\y_M, \dots ,\y_1)\in (\mathcal{Z} \times \mathcal{Y}'),
        \end{multline}
where $S_0\subset J\times J, |S_0|< +\infty$, and $\{\gamma_{j_1 j_2}\in\R , (j_1,j_2)\in S_0\}$ is a set of non-zero values which might depend on the choice of $\mathcal{Z}$ and $\mathcal{Y}'$ , where the latter depends on the choice of full measure $\mathcal{Y}$. 

From (b1), the set $S_0$ contains two different indices $(j_1,j_2)$ and $(j_1',j_2')$, with $j_2\neq 'j_2$. To see this, assume we have $j_2=j_2'$. Then we can group $v_{j_2}(\z_2 , \z_1 , \y_M, \dots, \y_2 )$
    \begin{multline}
        v_{j_2}(\z_2 , \z_1 , \y_M, \dots, \y_2 ) \sum_{j_1:(j_1,j_2)\in S_0} \gamma_{j_1 j_2}  v_{j_1}(\z_1 , \y_M, \dots , \y_1 ) = 0, \\ \forall (\z_2 ,\z_1 ,\y_M, \dots ,\y_1)\in (\mathcal{Z} \times \mathcal{Y}'),
        \end{multline}
Given that we have at least two indices $(j_1,j_2)$ and $(j_1',j_2)$ with $j_1\neq j_1'$, the above equation contradicts (b4), as the family $V_J$ contains linear independent functions under finite mixtures for every $\mathcal{Y}'\subset  \mathcal{Y}$, with $\mu(\mathcal{Y}) = \mu(\times_{m=1}^M \R^d)$.

Now we define $J_0:=\{j_2\subset J | \exists (j_1,j_2)\in S_0\}$, $|J_0|< +\infty$. Then, linear dependence can occur for any $(\bm{\beta}_M,\dots,\bm{\beta}_1) \in (\pi(\mathcal{Z}, \{1\}) \times \pi(\mathcal{Y}', \{M, \dots, 2\}))$, where $\pi$ denote set projections defined in Def. \ref{def:set-projection}.
 \begin{multline}
 \sum_{j_2\in J_0}\Bigg( \sum_{j_1:(j_1,j_2)\in S_0} \gamma_{j_1 j_2}  v_{j_1}(\z_1=\bm{\beta}_{M} , \y_M =\bm{\beta}_{M-1}, \dots , \y_2=\bm{\beta}_{1},\y_1 ) \Bigg) \\ v_{j_2}(\z_2 , \z_1 = \bm{\beta}_M , \y_M=\bm{\beta}_{M-1}, \dots, \y_2 = \bm{\beta}_1 )= 0, \\  \forall (\z_2 ,\y_1)\in (\mathcal{Z}_2 \times \mathcal{Y}'_1).
 \end{multline}
Where both $\mathcal{Z}_2=\mathcal{Z}_{(\bm{\beta}_M)}$ and $\mathcal{Y}'_1 = \mathcal{Y}'_{(\bm{\beta}_{M-1}, \dots, \bm{\beta}_1 )}$ are reverse projections (Def. \ref{def:reverse-projection}) from $\mathcal{Z}$ and $\mathcal{Y}'$ respectively  and are dependent on $\bm{\beta}_M, \dots, \bm{\beta}_1$. We note these sets are never empty sets. Now, we define the following set.
\begin{multline}
\mathcal{D}_0 := \Bigg\{(\bm{\beta}_M,\dots,\bm{\beta}_1) \in \Big(\pi(\mathcal{Z}, \{1\}) \times \pi(\mathcal{Y}', \{M, \dots, 2\})\Big)\bigg| \\ \sum_{j_1:(j_1,j_2)\in S_0}\gamma_{j_1,j_2}v_{j_1}(\z_1=\bm{\beta}_{M} , \y_M =\bm{\beta}_{M-1}, \dots , \y_2=\bm{\beta}_{1},\y_1 ) = 0, \forall\y_1\in\mathcal{Y}'_{(\bm{\beta}_{M-1}, \dots, \bm{\beta}_1 )}\Bigg\}.
\end{multline}
From (b4) we require linear independence for every $\mathcal{Y}' \subset \mathcal{Y}$, where $\mathcal{Y}$ has full measure; and every $\mathcal{Z}\subset \R^d$. Therefore, linear dependence as described above can happen at most in $(\R^d \times (\times_{m=1}^M \R^d) \setminus\mathcal{Y})$, which is zero-measured as $((\times_{m=1}^M \R^d) \setminus\mathcal{Y})$ has zero measure in $(\times_{m=1}^M \R^d)$. Therefore, from (b4) we know the set $\mathcal{D}_0 \subset (\pi(\mathcal{Z}, \{1\})  \times \mathcal{Y}_0)$ has zero measure due to $\mathcal{Y}_0 \subset \pi(\mathcal{Y}', \{M, \dots, 2\})$ having zero measure.

Now define $\mathcal{D}:= (\pi(\mathcal{Z}, \{1\}) \times \pi(\mathcal{Y}', \{M, \dots, 2\}) ) \setminus \mathcal{D}_0$ which is non-zero measured. From assumption (b1), we have $\forall (\z_1=\bm{\beta}_{M} , \y_M =\bm{\beta}_{M-1}, \dots , \y_2=\bm{\beta}_{1})\in \mathcal{D}$, there exists $\y_1\in \mathcal{Y}'_{(\bm{\beta}_{M-1}, \dots, \bm{\beta}_1 )}$ such that 
$$\tilde{\gamma}_{j_2}(\bm{\beta}_M,\dots,\bm{\beta}_1):=\sum_{j_1:(j_1,j_2)\in S_0}\gamma_{j_1j_2}v_{j_1}(\z_1=\bm{\beta}_{M} , \y_M =\bm{\beta}_{M-1}, \dots , \y_2=\bm{\beta}_{1},\y_1)\neq 0$$for at least two $j_2\in J_0$. This implies linear dependence of $\{v_{j_2}(\z_2 , \z_1 = \bm{\beta}_M , \y_M=\bm{\beta}_{M-1}, \dots, \y_2 = \bm{\beta}_1 )\}$ on $\z_2\in \mathcal{Z}_{(\bm{\beta}_M)}$, $\forall (\z_1 = \bm{\beta}_M , \y_M=\bm{\beta}_{M-1}, \dots, \y_2 = \bm{\beta}_1 )\in \mathcal{D}$.

Under assumption (b5), we can split $J_0$ into subsets indexed by $k\in K(\bm{\beta}_M,\dots,\bm{\beta}_1)$, such that the functions within each subset $J_k(\bm{\beta}_M,\dots,\bm{\beta}_1)$ are equal
\begin{multline}
J_0 = \cup_{k \in K(\bm{\beta}_M,\dots,\bm{\beta}_1)} J_k(\bm{\beta}_M,\dots,\bm{\beta}_1), \text{  }
J_k(\bm{\beta}_M,\dots,\bm{\beta}_1) \cap J_{k'}(\bm{\beta}_M,\dots,\bm{\beta}_1) = \emptyset, \quad  \forall k \neq k' \in K(\bm{\beta}_M,\dots,\bm{\beta}_1),
\end{multline}
\begin{multline}
 j \neq j' \in J_k(\bm{\beta}_M,\dots,\bm{\beta}_1) \quad \Leftrightarrow \quad \\
v_{j}(\z_2 , \z_1 = \bm{\beta}_M , \dots, \y_2 = \bm{\beta}_1 ) = v_{j'}(\z_2 , \z_1 = \bm{\beta}_M , \dots, \y_2 = \bm{\beta}_1 ), \quad \forall \z_2 \in \mathcal{Z}_{(\bm{\beta}_M)}.
\end{multline}

Then we can rewrite the linear dependence condition for any $(\bm{\beta}_M,\dots,\bm{\beta}_1) \in \mathcal{D}$ as
    \begin{multline}
        \sum_{k \in K(\bm{\beta}_M,\dots,\bm{\beta}_1)} \Bigg( \sum_{j_2 \in J_k(\bm{\beta}_M,\dots,\bm{\beta}_1)}  \sum_{j_1 : (j_1,j_2)\in S_0} \gamma_{j_1 j_2}  v_{j_1}(\z_1=\bm{\beta}_{M} , \y_M =\bm{\beta}_{M-1}, \dots , \y_2=\bm{\beta}_{1},\y_1 ) \\
         v_{j_2}(\z_2 , \z_1 = \bm{\beta}_M , \dots, \y_2 = \bm{\beta}_1 ) \Bigg)= 0, \quad  \forall (\z_2, \y_1)\in ,\mathcal{Z}_{(\bm{\beta}_M)} \times \mathcal{Y}'_{(\bm{\beta}_{M-1}, \dots, \bm{\beta}_1 )}
        \end{multline}
Recall from that $v_{j}(\z_2 , \z_1 = \bm{\beta}_M , \dots, \y_2 = \bm{\beta}_1 )$ and $v_{j'}(\z_2 , \z_1 = \bm{\beta}_M , \dots, \y_2 = \bm{\beta}_1 )$ are the same functions on $\z_2 \in \mathcal{Z}_{(\bm{\beta}_M)} $ iff.~$j \neq j'$ are in the same index set $J_k(\bm{\beta}_M,\dots,\bm{\beta}_1)$. This means if linear independence holds, then for any $(\bm{\beta}_M,\dots,\bm{\beta}_1)\in \mathcal{D}$, under assumptions (b1) and (b5),
\begin{multline}
\sum_{j_2 \in J_k(\bm{\beta}_M,\dots,\bm{\beta}_1)}  \sum_{j_1 : (j_1,j_2)\in S_0} \gamma_{j_1 j_2}  v_{j_1}(\z_1=\bm{\beta}_{M} , \y_M =\bm{\beta}_{M-1}, \dots , \y_2=\bm{\beta}_{1},\y_1 ) = 0, \\ 
\quad \forall  \y_1\in ,\mathcal{Y}'_{(\bm{\beta}_{M-1}, \dots, \bm{\beta}_1 )}, \quad k \in K(\bm{\beta}_{M-1}, \dots, \bm{\beta}_1 ).  
\end{multline}

Define $C(\bm{\beta}_M,\dots,\bm{\beta}_1) = \min_k |J_k(\bm{\beta}_M,\dots,\bm{\beta}_1)|$ the minimum cardinality count for $j_2$ indices in the $J_k(\bm{\beta}_M,\dots,\bm{\beta}_1)$ subsets. 

Choose $(\bm{\beta}_M^*,\dots,\bm{\beta}_1^*) \in \arg\min_{(\bm{\beta}_M,\dots,\bm{\beta}_1) \in \mathcal{D}} C(\bm{\beta}_M,\dots,\bm{\beta}_1) $:
\begin{itemize}
    \item[1.] We have $C(\bm{\beta}_M^*,\dots,\bm{\beta}_1^*)  < |J_0|$ and $|K(\bm{\beta}_M^*,\dots,\bm{\beta}_1^*)| \geq 2$. Otherwise for all $j \neq j' \in J_0$ we have $v_{j}(\z_2 , \z_1 = \bm{\beta}_M , \dots, \y_2 = \bm{\beta}_1 ) = v_{j'}(\z_2 , \z_1 = \bm{\beta}_M , \dots, \y_2 = \bm{\beta}_1 )$ for all $\z_2 \in \mathcal{Z}_{(\bm{\beta}_M)}$ and $(\bm{\beta}_M,\dots,\bm{\beta}_1) \in\mathcal{D})$, so that they are linearly dependent on $(\z_2, \z_1, \y_M, \dots, \y_2) \in \tilde{\mathcal{Z}}_2 \times \mathcal{D}$ for some non-zero measure $\tilde{\mathcal{Z}}_2 \subseteq \pi(\mathcal{Z}, \{2\})$, a contradiction to assumption (b4) by setting $\mathcal{Y}' = \mathcal{D}$, which holds no matter the choice of $\mathcal{Z}$ as $\mathcal{Y}'\subset \mathcal{Y}$ where $\mathcal{Y}$ has full measure.
    % yw / this is true because eq 9 must be true for all beta (not just in a single instantiation of beta). if J_k has only a single item, then the inner summation must be zero, and it violates the linear independence of u_i since it has to be true for all beta.
    \item[2.] Now assume $|J_{1}(\bm{\beta}_M^*,\dots,\bm{\beta}_1^*)| = C(\bm{\beta}_M^*,\dots,\bm{\beta}_1^*)$ w.l.o.g.. From assumption (b5), we know that for any $j \in J_1(\bm{\beta}_M^*,\dots,\bm{\beta}_1^*)$ and $j' \in J_0 \backslash J_{1}(\bm{\beta}_M^*,\dots,\bm{\beta}_1^*)$, $v_{j}(\z_2 , \z_1 = \bm{\beta}_M , \dots, \y_2 = \bm{\beta}_1 ) = v_{j'}(\z_2 , \z_1 = \bm{\beta}_M , \dots, \y_2 = \bm{\beta}_1 )$ only on a zero measure subset of $\mathcal{Z}_{(\bm{\beta}_M)}$ at most. Then as $|J_0| < +\infty$ and $\mathcal{Z}_{(\bm{\beta}_M)} \subset \mathbb{R}^{d}$ has non-zero measure, there exist $\z_0 \in \mathcal{Z}_{(\bm{\beta}_M)}$ and $\delta > 0$ such that 
    \begin{multline}
        |v_{j}(\z_2 = \z_0 , \z_1 = \bm{\beta}_M^* , \dots, \y_2 = \bm{\beta}_1^* ) - v_{j'}(\z_2 = \z_0 , \z_1 = \bm{\beta}_M^* , \dots, \y_2 = \bm{\beta}_1^* )| \geq \delta, \quad \\
        \forall j \in J_{1}(\bm{\beta}_M^*,\dots,\bm{\beta}_1^*), \forall j' \in J_0 \backslash J_{1}(\bm{\beta}_M^*,\dots,\bm{\beta}_1^*)
    \end{multline}
    Under assumption (b6), there exists $\epsilon(j) > 0$ such that we can construct an $\epsilon$-ball $B_{\epsilon(j)}(\bm{\beta}_M^*,\dots,\bm{\beta}_1^*)$ using $\ell_2$-norm, such that 
    \begin{multline}
     |v_{j}(\z_2 = \z_0 , \z_1 = \bm{\beta}_M^* , \dots, \y_2 = \bm{\beta}_1^* ) - v_{j}(\z_2 = \z_0 , \z_1 = \bm{\beta}_M , \dots, \y_2 = \bm{\beta}_1)| \leq \delta/3, \\
    \quad \forall (\bm{\beta}_M,\dots,\bm{\beta}_1) \in B_{\epsilon(j)}(\bm{\beta}_M^*,\dots,\bm{\beta}_1^*).
    \end{multline}
    Choosing a suitable $0 < \epsilon \leq \min_{j \in J_0} \epsilon(j)$ (note that $\min_{j \in J_0} \epsilon(j) > 0$ as $|J_0| < +\infty$) and constructing an $\ell_2$-norm-based $\epsilon$-ball $B_{\epsilon}(\bm{\beta}_M^*,\dots,\bm{\beta}_1^*) \subset \mathcal{Y}_1$, we have for all $j \in J_{1}(\bm{\beta}_M^*,\dots,\bm{\beta}_1^*), j' \in J_0 \backslash J_{1}(\bm{\beta}_M^*,\dots,\bm{\beta}_1^*)$, $j' \notin J_1(\bm{\beta}_M,\dots,\bm{\beta}_1)$ for all $(\bm{\beta}_M,\dots,\bm{\beta}_1) \in B_{\epsilon}(\bm{\beta}_M^*,\dots,\bm{\beta}_1^*)$ due to
    \begin{multline}
     |v_{j}(\z_2 = \z_0 , \z_1 = \bm{\beta}_M , \dots, \y_2 = \bm{\beta}_1 ) - v_{j'}(\z_2 = \z_0 , \z_1 = \bm{\beta}_M , \dots, \y_2 = \bm{\beta}_1)| \geq \delta/3,   \\
      \quad \forall (\bm{\beta}_M,\dots,\bm{\beta}_1) \in B_{\epsilon}(\bm{\beta}_M^*,\dots,\bm{\beta}_1^*).
    \end{multline}
    So this means for the split $\{ J_k(\bm{\beta}_M,\dots,\bm{\beta}_1) \}$ of any $(\bm{\beta}_M,\dots,\bm{\beta}_1) \in B_{\epsilon}(\bm{\beta}_M^*,\dots,\bm{\beta}_1^*)$, we have $ J_1(\bm{\beta}_M,\dots,\bm{\beta}_1) \subset J_1(\bm{\beta}^*)$ 
    and therefore $ |J_1(\bm{\beta}_M,\dots,\bm{\beta}_1)| \leq |J_1(\bm{\beta}_M^*,\dots,\bm{\beta}_1^*)|$. Now by definition of $(\bm{\beta}_M^*,\dots,\bm{\beta}_1^*)\in \arg\min_{(\bm{\beta}_M,\dots,\bm{\beta}_1) \in \mathcal{Y}} C(\bm{\beta}_M,\dots,\bm{\beta}_1)$ and $|J_{1}(\bm{\beta}_M^*,\dots,\bm{\beta}_1^*)| = C(\bm{\beta}_M^*,\dots,\bm{\beta}_1^*)$, we have $J_1(\bm{\beta}_M,\dots,\bm{\beta}_1) = J_1(\bm{\beta}_M^*,\dots,\bm{\beta}_1^*)$ for all $(\bm{\beta}_M,\dots,\bm{\beta}_1) \in B_{\epsilon}(\bm{\beta}_M^*,\dots,\bm{\beta}_1^*)$. 
    \item[3.] One can show that $|J_{1}(\bm{\beta}_M^*,\dots,\bm{\beta}_1^*)| = 1$, otherwise by definition of the splits and the above point, there exists $j \neq j' \in J_{1}(\bm{\beta}_M^*,\dots,\bm{\beta}_1^*)$ such that $v_{j}(\z_2 , \z_1 = \bm{\beta}_M , \dots, \y_2 = \bm{\beta}_1 ) = v_{j'}(\z_2 , \z_1 = \bm{\beta}_M , \dots, \y_2 = \bm{\beta}_1 )$ for all $\z_2 \in \mathcal{Z}_{(\bm{\beta}_M)}$ and $(\bm{\beta}_M,\dots,\bm{\beta}_1) \in B_{\epsilon}(\bm{\beta}_M^*,\dots,\bm{\beta}_1^*)$, a contradiction to assumption (b4) by setting $\mathcal{Y}' = B_{\epsilon}(\bm{\beta}_M^*,\dots,\bm{\beta}_1^*)$. 
    Now assume that $j \in J_1(\bm{\beta}_M^*,\dots,\bm{\beta}_1^*)$ is the only index in the subset, then the fact proved in the above point that $J_1(\bm{\beta}_M,\dots,\bm{\beta}_1) = J_1(\bm{\beta}_M^*,\dots,\bm{\beta}_1^*)$ for all $(\bm{\beta}_M,\dots,\bm{\beta}_1) \in B_{\epsilon}(\bm{\beta}_M^*,\dots,\bm{\beta}_1^*)$ means
    \begin{multline}
    \sum_{ j_1: (j_1, j_2) \in S_0} \gamma_{j_1j_2} v_{j_1}(\z_1=\bm{\beta}_{M} , \y_M =\bm{\beta}_{M-1}, \dots , \y_2=\bm{\beta}_{1},\y_1 )  = 0, \\
    \quad \forall \y_1\in ,\mathcal{Y}'_{(\bm{\beta}_{M-1}, \dots, \bm{\beta}_1 )}, \quad \forall (\bm{\beta}_M, \dots, \bm{\beta}_1)) \in B_{\epsilon}(\bm{\beta}_M^*, \dots, \bm{\beta}_1^*),
    \end{multline}
    again a contradiction to assumption (b4) by setting $\mathcal{D} = B_{\epsilon}(\bm{\beta}_M^*, \dots, \bm{\beta}_1^*)$. 
\end{itemize}

The above 3 points indicate that linear dependence cannot hold for all $(\bm{\beta}_M,\dots,\bm{\beta}_1) \in\mathcal{D}$, and thus reaching a contradiction within the projection sets from the chosen $\mathcal{Z}$ and $\mathcal{Y}'\subset\mathcal{Y}$. Given that the space in which we define $\mathcal{Y}'$ is finite-dimensional, we can cover the entire space with epsilon balls of certain radius $\epsilon$.  We apply the same logic to the full measure set from (b4) on the family $\{v_{j_2}(\z_2, \z_1, \y_M, \dots, \y_2) : j_2 \in J\}$. Therefore, we can cover the entire full measure set with $\epsilon$ balls such that the previous arguments hold for any non-zero measure sets $(\pi(\mathcal{Z}, \{1\}) \times \pi(\mathcal{Y}', \{M,\dots, 2\})) \subseteq \mathcal{Y}$ except for some zero measure set of points, no matter the choice of $\mathcal{Z}$ or $\mathcal{Y}'$ ($\mathcal{Y}$ is the full measure set in (b4)). 

Finally, the above argument needs to hold  for any full measure set $\mathcal{Y}\subset (\times_{m=1}^M \R^d)$. However, from (b4), the above condition implies $\pi(\mathcal{Y}', \{M,\dots, 2\})$ can only satisfy linear dependence if it is zero-measured, which is a direct contradiction the statement.

    \paragraph{Case $n > 2$.} Now assume the statement holds for $n-1$, and again prove the case by contradiction. Assume the statement is false for $n > 2$. Then, for every set $\mathcal{Y}\subset (\times_{m=1}^M \R^d)$ with $\mu(\mathcal{Y}) = \mu(\times_{m=1}^M \R^d)$, there exists non-zero measure sets $\mathcal{Z} \subset (\times_{m=1}^n \R^d)$ and $\mathcal{Y}'\subset \mathcal{Y}$ such that the family $(\otimes_{m=1}^n \mathcal{V}_J)$ contains linear dependent functions. In other words, there exists $S_0 \subset (\times_{m=1}^{n}J), |S_0| < +\infty$ such that
    \begin{multline}
    \sum_{(j_n,\dots j_1)\in S_0}\gamma_{j_n,\dots,j_1}v_{j_n}(\z_n,\dots,\z_1,\y_M,\dots, \y_{n})\dots v_{j_1}(\z_1,\y_M,\dots,\y_1) = 0, \\ \forall (\z_n,\dots ,\z_1 ,\y_M, \dots ,\y_1)\in (\mathcal{Z} \times \mathcal{Y}'),
    \end{multline}
    with $\{\gamma_{j_n, \dots, j_1}\in\R , (j_n,\dots,j_1)\in S_0\}$ a set of non-zero values which, as before, might depend on the choice of $\mathcal{Z}$ and $\mathcal{Y}'\subset\mathcal{Y}$. The previous equality can be arranged as follows
\begin{multline}
    \sum_{(j_n,\dots j_1)\in S_0}\gamma_{j_n,\dots,j_1}v_{j_n}(\z_n,\dots,\z_1,\y_M,\dots, \y_{n}) v_{j_{n-1},\dots,j_1}(\z_{n-1},\dots,\z_1,\y_M,\dots,\y_1) = 0, \\ \forall (\z_n,\dots ,\z_1 ,\y_M, \dots ,\y_1)\in (\mathcal{Z} \times \mathcal{Y}'),
\end{multline}
    where $v_{j_{n-1},\dots,j_1}:=v_{j_{n-1}}(\z_{n-1},\dots,\z_1,\y_M,\dots, \y_{n-1})\dots v_{j_1}(\z_1,\y_M,\dots,\y_1)$ for $j_i \in J, 1\leq i < n$ denotes the functions on the family $(\otimes_{m=1}^{n-1} \mathcal{V}_J)$ and satisfies the following: 
    \begin{enumerate}
        \item[(b1)] Positive function values: $v_{j_{n-1},\dots,j_1}(\z_{n-1},\dots,\z_1,\y_M,\dots,\y_1) > 0$ for all $(j_{n-1}, \dots, j_1) \in (\times_{m=1}^{n-1} J), (\z_{n-1},\dots,\z_1,\y_M,\dots,\y_1) \in (\times_{m=1}^{M+n-1} \R^{d})$.
        \item[(b4)] There exist non-zero measure sets $\mathcal{Y}\subset (\times_{m=1}^M \R^d)$ with $\mu(\mathcal{Y}) = \mu(\times_{m=1}^M \R^d)$ such that for every non-zero measure sets $\mathcal{Y}' \subset \mathcal{Y}$, and $\mathcal{Z} \subset (\times_{m=1}^{n-1} \R^{d})$, the family $(\otimes_{m=1}^{n-1} \mathcal{V}_J)$ contains linear independent functions under finite mixtures in $(\z_{n-1},\dots,\z_1,\y_M,\dots,\y_1)\in \mathcal{Z} \times \mathcal{Y}'$.
    \end{enumerate}
The strategy here is to reduce this case to the above base case ($n=2$), where importantly we need to show: 
\begin{enumerate}
    \item[(1)]  $S_0$ contains at least two $(j_n, \dots, j_1)$ and $(j_n',\dots, j_1')$ with $j_n \neq j_n'$; and
    \item[(2)] the overlapping variables between the product of families, $(\z_{n-1}, \dots, \z_1, \y_M, \dots, \y_n) \in (\pi(\mathcal{Z}, \{n-1, \dots, 1\}) \times \pi(\mathcal{Y}', \{M, \dots, n\}))$ do not cause linear dependence in non-zero measure sets for any choice of $\mathcal{Z}$ and $\mathcal{Y}'\subset \mathcal{Y}$.
\end{enumerate}
We can simply see (i) holds from (b1) and (b4). Therefore, we define $J_0:=\{j_n\in J | \exists (j_n, \dots,j_1)\in S_0\}$, $|J_0|< +\infty$. Then, linear dependence can occur for any $(\bm{\beta}_M,\dots,\bm{\beta}_1) \in (\pi(\mathcal{Z}, \{n-1, \dots, 1\}) \times \pi(\mathcal{Y}', \{M, \dots, n\}))\subset (\times_{m=1}^{M}\R^d)$, where $\pi$ denote set projections defined in Def. \ref{def:set-projection}.
 \begin{multline}
 \sum_{j_n\in J_0} v_{j_n}(\z_n, \z_{n-1}=\bm{\beta}_{M},\dots,\z_1=\bm{\beta}_{M-n+2},\y_M,=\bm{\beta}_{M-n+1}\dots, \y_{n}=\bm{\beta}_1) \\ \Bigg( \sum_{(j_{n-1},\dots,j_1):(j_n,\dots,j_1)\in S_0} \gamma_{j_n,\dots, j_1} v_{j_{n-1},\dots,j_1}(\z_{n-1}=\bm{\beta}_{M},\dots,\z_1=\bm{\beta}_{M-n+2},\y_M,=\bm{\beta}_{M-n+1}, \\ \dots, \y_{n}=\bm{\beta}_1, \y_{n-1},\dots,\y_1) \Bigg) = 0,  \\ \forall (\z_n ,\y_{n-1}, \dots, \y_1)\in \Big(\mathcal{Z}_{(\bm{\beta}_M, \dots, \bm{\beta}_{M-n+2})} \times \mathcal{Y}'_{(\bm{\beta}_{M-n+1}, \dots, \bm{\beta}_1 )}\Big).
 \end{multline}
Where as in the base case, both $\mathcal{Z}_{(\bm{\beta}_M, \dots, \bm{\beta}_{M-n+2})} \subset \R^d$ and $\mathcal{Y}'_{(\bm{\beta}_{M-n+1}, \dots, \bm{\beta}_1)} \subset (\times_{m=1}^{n-1}\R^d) $ are reverse projections (Def. \ref{def:reverse-projection}) from $\mathcal{Z}$ and $\mathcal{Y}'$ respectively  and are dependent on $\bm{\beta}_M, \dots, \bm{\beta}_1$. As before, these sets are never empty, and we can establish the following equivalence between variables and sets (given the choice of $\mathcal{Z}$ and $\mathcal{Y}'\subset \mathcal{Y}$) with respect to the base case.
\begin{itemize}
    \item $\z_{n}$ is equivalent to $\z_2$, and thus $\mathcal{Z}_{(\bm{\beta}_M, \dots, \bm{\beta}_{M-n+2})} \subset \R^d$ is equivalent to $\mathcal{Z}_{(\bm{\beta}_M, \dots, \bm{\beta}_{M})} \subset \R^d$.
    \item $\z_{n-1}, \dots, \z_1$ are equivalent to $\z_1$, and thus $\pi(\mathcal{Z}, \{n-1, \dots, 1\}) \subset (\times_{m=1}^{n-1}\R^d)$ is equivalent to $\pi(\mathcal{Z}, \{1\}) \subset \R^d$.
    \item $\y_{M}, \dots, \y_n$ are equivalent to $\y_M, \dots, \y_2$, and thus $\pi(\mathcal{Y}', \{M, \dots, n\}) \subset (\times_{m=1}^{M-n+1}\R^d)$ is equivalent to $\pi(\mathcal{Y}', \{M, \dots, 2\}) \subset (\times_{m=1}^{M-1}\R^d)$.
    \item $\y_{n-1}, \dots, \y_1$ are equivalent to $\y_1$, and thus  $\mathcal{Y}'_{(\bm{\beta}_{M-n+1}, \dots, \bm{\beta}_1)} \subset (\times_{m=1}^{n-1}\R^d) $ is equivalent to  $\mathcal{Y}'_{(\bm{\beta}_{M-1}, \dots, \bm{\beta}_1)} \subset \R^d $.
\end{itemize}
Given the above equivalences, the arguments for contradiction given linear dependence for any $(\bm{\beta}_M,\dots,\bm{\beta}_1) \in (\pi(\mathcal{Z}, \{n-1, \dots, 1\}) \times \pi(\mathcal{Y}', \{M, \dots, n\}))\subset (\times_{m=1}^{M}\R^d)$ still hold. The only difference is that the dimensionality of the projections and reverse projections will change (except for $\z_n\in\mathcal{Z}_{(\bm{\beta}_M, \dots, \bm{\beta}_{M-n+2})}$). Therefore, the set
\begin{multline}
\mathcal{D}_0 := \Bigg\{(\bm{\beta}_M,\dots,\bm{\beta}_1) \in \Big(\pi(\mathcal{Z}, \{n-1, \dots, 1\}) \times \pi(\mathcal{Y}', \{M, \dots, n\})\Big)\bigg| \\ \sum_{(j_{n-1},\dots,j_1):(j_n,\dots,j_1)\in S_0} \gamma_{j_n,\dots, j_1} v_{j_{n-1},\dots,j_1}(\z_{n-1}=\bm{\beta}_{M},\dots,\z_1=\bm{\beta}_{M-n+2},\y_M,=\bm{\beta}_{M-n+1}, \\ \dots, \y_{n}=\bm{\beta}_1, \y_{n-1},\dots,\y_1)  = 0, \forall(\y_{n-1}, \dots, \y_1)\in\mathcal{Y}'_{(\bm{\beta}_{M-n+1}, \dots, \bm{\beta}_1)}\Bigg\}
\end{multline}
has zero measure under assumption (b4), which implies that for any $(\z_{n-1}=\bm{\beta}_{M},\dots, \z_1=\bm{\beta}_{M-n+2},\y_M,=\bm{\beta}_{M-n+1}, \dots, \y_{n}=\bm{\beta}_1)\in \mathcal{D}$, with $\mathcal{D}:= (\pi(\mathcal{Z}, \{n-1, \dots, 1\}) \times \pi(\mathcal{Y}', \{M, \dots, n\})) \setminus \mathcal{D}_0$, there exists $(\y_{n-1}, \dots, \y_1)\in\mathcal{Y}'_{(\bm{\beta}_{M-n+1}, \dots, \bm{\beta}_1)}$ non-zero-measured, we have
\begin{multline}
    \tilde{\gamma}_{j_n}(\bm{\beta}_M,\dots,\bm{\beta}_1):=\sum_{(j_{n-1},\dots,j_1):(j_n,\dots,j_1)\in S_0} \gamma_{j_n,\dots, j_1} v_{j_{n-1},\dots,j_1}(\z_{n-1}=\bm{\beta}_{M},\dots,\\ \z_1=\bm{\beta}_{M-n+2},\y_M,=\bm{\beta}_{M-n+1}, \dots, \y_{n}=\bm{\beta}_1, \y_{n-1},\dots,\y_1)\neq 0
\end{multline}
for at least two $j_n\in J_0$. As before, this implies linear dependence of $\{v_{j_n}(\z_n, \z_{n-1}=\bm{\beta}_{M},\dots,\z_1=\bm{\beta}_{M-n+2},\y_M,=\bm{\beta}_{M-n+1}\dots, \y_{n}=\bm{\beta}_1)\}$ on $\z_n\in\mathcal{Z}_{(\bm{\beta}_M, \dots, \bm{\beta}_{M-n+2})}$, $\forall ( \z_{n-1}=\bm{\beta}_{M},\dots,\z_1=\bm{\beta}_{M-n+2},\y_M,=\bm{\beta}_{M-n+1}\dots, \y_{n}=\bm{\beta}_1)\in \mathcal{D}$. Therefore, under assumptions (b5-b6) we can show linear dependence cannot hold for all $(\bm{\beta}_M,\dots,\bm{\beta}_1) \in\mathcal{D}$, no matter the choice of $\mathcal{Z}$ and $\mathcal{Y}'\subset \mathcal{Y}$, following the equivalent arguments given in the case with $n=2$.
\end{proof}

Regarding (b5), we observe that linear dependence on $M$ products of functions given fixed subsets of variables does not imply repeating functions anymore. Instead, we show that at least one component of the product family is linear independent under finite mixtures.

\begin{lemma}\label{lemma:extension_b5}
\textbf{B5 Extension.} Assume a distribution family $\mathcal{V}_J=\{v_j(\z,\y_M,\dots,\y_1)| j\in J\}$ with variables $\z,\y_M,\dots,\y_1$ defined on $\times_{m=1}^{M+1} \R^{d}$, such that it satisfies (b1), (b4), (b5), (b6).

Then, for any $n$ with $2\leq n \leq M$, any $(\bm{\beta}_M,\dots,\bm{\beta}_1) \in ( \times_{m=1}^M \R^{d})$, any non-zero measure $\mathcal{Z}\subset (\times_{m=1}^{n}\R^{d})$, and $(\times_{m=1}^n J_m) \subset (\times_{m=1}^n J)$, then $\{v_{j_n}(\z_n,\dots,\z_1,\y_M=\bm{\beta}_M,\dots, \y_n=\bm{\beta}_n)\dots v_{j_1}(\z_1,\y_M=\bm{\beta}_M,\dots, \y_1=\bm{\beta}_1)|(j_n,\dots,j_1)\in (\times_{m=1}^n J_m)\}$ contains linearly dependent functions only if $\exists m \leq n$ such that $\{v_{j_m}(\z_m, \dots, \z_1,\y_M=\bm{\beta}_M,\dots, \y_m=\bm{\beta}_m)|j_m\in J_m\}$ contains linearly dependent functions on $(\z_m, \dots, \z_1) \in \pi(\mathcal{Z},\{m, \dots, 1\})$.
\end{lemma}
\begin{proof}
We proof the statement by induction starting from the base case ($n=2$) as follows.

\paragraph{Case $n=2$.} Assume the above necessity assertion is false. Fix $\mathcal{Z}\subset \R^{d}\times \R^{d}$, and assume that for any $(\bm{\beta}_M,\dots,\bm{\beta}_1) \in ( \times_{m=1}^M \R^{d})$, and $S_0:= J_1\times J_2\subset J\times J$, the family  $\{v_{j_1}(\z_1,\y_M=\bm{\beta}_M,\dots, \y_1=\bm{\beta}_1)v_{j_2}(\z_2, \z_1,\y_M=\bm{\beta}_{M},\dots, \y_2=\bm{\beta}_2)|(j_1,j_2)\in (J_1,J_2)\}$ contains linear dependent functions. %Furthermore, assume the family $\{v_{j_1}(\z_1,\y_M=\bm{\beta}_M,\dots, \y_1=\bm{\beta}_1)|j_1\in J_1\}$ contains linearly independent functions on $\z_1 \in \pi(\mathcal{Z},\{1\})$. 
From (b1, b4-b6), and Lemma \ref{lemma:linear_independence_two_nonlinear_gaussians}, by setting $u_i(\y=\z_1, \x=\emptyset)=v_{j_1}(\z_1,\y_M=\bm{\beta}_M,\dots, \y_1=\bm{\beta}_1)$, $i=j_1$, and $v_j=(\z=\z_2, \y=\z_1)=v_{j_2}(\z_2, \z_1,\y_M=\bm{\beta}_{M},\dots, \y_2=\bm{\beta}_2)$, $j=j_2$, we know that either one of the families contains linear dependent functions (from contradictions to (a3) or (a4)). %Therefore, $\{v_{j_1}(\z_1,\y_M=\bm{\beta}_M,\dots, \y_1=\bm{\beta}_1)|j_1\in J_1\}$ contains linearly dependent functions on $\z_1 \in \pi(\mathcal{Z},\{1\})$.

\paragraph{Case $n > 2$.}

Again assume the necessity assertion is false. Fix $\mathcal{Z}\subset (\times_{m=1}^{n}\R^{d})$, and assume that for any $(\bm{\beta}_M,\dots,\bm{\beta}_1) \in ( \times_{m=1}^M \R^{d})$, and $(\times_{m=1}^n J_m) \subset (\times_{m=1}^n J)$, the family $\{v_{j_n}(\z_n,\dots,\z_1,\y_M=\bm{\beta}_M,\dots, \y_n=\bm{\beta}_n)\dots v_{j_1}(\z_1,\y_M=\bm{\beta}_M,\dots, \y_1=\bm{\beta}_1)|(j_n,\dots,j_1)\in (\times_{m=1}^n J_m)\}$ contains linearly dependent functions. 
%Furthermore, assume the family $\{v_{j_{n-1}}(\z_{n-1},\dots,\z_1,\y_M=\bm{\beta}_M,\dots, \y_{n-1}=\bm{\beta}_{n-1})\dots v_{j_1}(\z_1,\y_M=\bm{\beta}_M,\dots, \y_1=\bm{\beta}_1)|(j_{n-1},\dots,j_1)\in (\times_{m=1}^{n-1} J_m)\}$ contains linearly independent functions on $(\z_{n-1},\dots,\z_1) \in \pi(\mathcal{Z},\{n-1,\dots,1\})$. 
As before, from (b1, b4-b6), and Lemma \ref{lemma:linear_independence_two_nonlinear_gaussians}, we know that by setting 
\begin{multline*}
u_i(\y=(\z_{n-1}, \dots, \z_1), \x=\emptyset)=v_{j_{n-1}}(\z_{n-1},\dots,\z_1,\y_M=\bm{\beta}_M,\dots, \y_{n-1}=\bm{\beta}_{n-1})\dots \\ v_{j_1}(\z_1,\y_M=\bm{\beta}_M,\dots, \y_1=\bm{\beta}_1), \quad i=(j_{n-1},\dots, j_1), 
\end{multline*}
and $v_j(\z=\z_{n}, \y=(\z_{n-1}, \dots, \z_1))=v_{j_n}(\z_n,\dots \z_1,\y_M=\bm{\beta}_M,\dots, \y_n=\bm{\beta}_n), \quad j=j_n$, either one of the families must be linearly dependent from contradictions to (a3) or (a4).
Therefore, either $\{v_{j_{n}}(\z_{n},\dots,\z_1,\y_M=\bm{\beta}_M,\dots, \y_{n}=\bm{\beta}_{n})|j_{n}\in J_n)\}$ contains linearly dependent functions on $(\z_{n},\dots,\z_1) \in \pi(\mathcal{Z},\{n,\dots,1\})$, or from the induction hypothesis there is some $m \leq n-1$ such that $\{v_{j_{m}}(\z_{m}, \dots,\z_1,\y_M=\bm{\beta}_M,\dots, \y_{m}=\bm{\beta}_m)|j_m\in J_m\}$ contains linearly dependent functions on $(\z_m,\dots, \z_1) \in \pi(\pi(\mathcal{Z},\{n-1, \dots, 1\}), \{m, \dots, 1\}) = \pi(\mathcal{Z},\{m, \dots, 1\})$.
\end{proof}

\subsubsection{Linear independence on M+1 products of functions}

Given that assumption (b5) does not extend similarly as (b4), we adapt Lemma \ref{lemma:linear_independence_two_nonlinear_gaussians} where we use \ref{lemma:extension_b5} to show linear independence when $V_J$ is defined as a product of $M$ functions with consecutive variables. 

\begin{lemma}
\label{lemma:absolutely_magical_lemma}
\textbf{Linear Independence on Product functions with $M$ consecutive Overlapping variables (LIPO-M).} Assume two families $\mathcal{U}_I=\{u_i(\y,\x)| i\in I\}$ and $\mathcal{V}_J=\{v_j(\z,\y)| j\in J\}$, with $\x\in\R^{d_x}$, $\y\in\R^{d_y}$, and $\z\in\R^{d_z}$, with $d_y=d\cdot M$, $d_z=d\cdot n$, and $1 \leq n \leq M$. We further assume:
\begin{itemize}
    \item[(a)] The family $\mathcal{U}_I$ satisfies assumptions (a1) and (a3).
    \item[(b)] Each element in the family $\mathcal{V}_J$ is defined as a product of $n$ functions, all of which belong to the same family $\mathcal{W}_K=\{w_k(\z_1,\y_M,\dots,\y_1)| k\in K\}$, with $\z_1\in\R^d$, $(\y_M, \dots, \y_1) \in (\times_{m=1}^M \R^d)$, and is expressed as follows
    $$
    \mathcal{V}_J := (\otimes_{m=1}^n \mathcal{W}_K) = \{w_{k_n}(\z_n, \dots, \z_1, \y_M)\dots w_{k_1}(\z_1,\y_M,\dots,\y_1)| (k_n,\dots, k_1)\in (\times_{m=1}^n K) \}
    $$
    where $J:= (\times_{m=1}^n K)$, and $\z:=(\z_n, \dots, \z_1)$, $\y:=(\y_M, \dots, \y_1)$ in $\mathcal{V}_J$. The family $\mathcal{W}_K$ satisfies assumptions (b1), (b4), (b5), and (b6).
%\begin{itemize}
%\item[(b1)] Positive function values: $w_k(\z_1, \y_M,\dots,\y_1) > 0$ for all $k \in K, (\z_1, \y_M,\dots,\y_1) \in (\times_{m=1}^{M+1} \mathbb{R}^{d})$.
%\item[(b4)] There exist a non-zero measure set $\mathcal{Y}\subset (\times_{m=1}^M \R^d)$ with $\mu(\mathcal{Y}) = \mu(\times_{m=1}^M \R^d)$ such that for every non-zero measure sets $\mathcal{Y}' \subset \mathcal{Y}$, and $\mathcal{Z} \subset \R^{d}$, $\mathcal{W}_K$ contains linearly independent functions under finite mixtures on $(\z_1,\y_M,\dots,\y_1) \in \mathcal{Z} \times \mathcal{Y}'$;
%\item[(b5)] For any $(\bm{\beta}_M,\dots,\bm{\beta}_1) \in ( \times_{m=1}^M \R^{d})$, any non-zero measure subset $\mathcal{Z} \subset \mathbb{R}^{d}$ and any subset $K_0 \subset K$ such that $|K_0| < +\infty$, $\{w_k(\z_1,\y_M=\bm{\beta}_M,\dots, \y_1=\bm{\beta}_1) | k \in K_0 \}$ contains linearly dependent functions on $\z \in \mathcal{Z}$ only if $ \exists \ k \neq k' \in K_0$ such that $w_k(\z_1,\y_M=\bm{\beta}_M,\dots, \y_1=\bm{\beta}_1) = w_{k'}(\z_1,\y_M=\bm{\beta}_M,\dots, \y_1=\bm{\beta}_1)$ for all $\z_1 \in \mathbb{R}^{d}$.
%\item[(b6)] Continuity: for any $k \in K$, $w_k(\z_1,\y_M,\dots, \y_1)$ is continuous in $(\y_M,\dots, \y_1) \in ( \times_{m=1}^M \R^{d})$.
%\end{itemize}
\end{itemize}
Then for any non-zero measure subset $\mathcal{Z} \subset \mathbb{R}^{d_z}$, $\mathcal{U}_I \otimes \mathcal{V}_J := \{u_i(\y, \x)v_{j}(\z, \y) | i \in I, j \in J \}$ 
contains linear indepedent functions under finite mixtures defined on $(\x, \y, \z) \in \mathcal{X} \times \mathbb{R}^{d_y}\times \mathcal{Z}$.
\end{lemma}

\begin{proof}
Assume this sufficiency statement is false, then there exist a non-zero measure subset $\mathcal{Z} \subset \mathbb{R}^{d_z}$, $S_0 \subset I \times J$ with $|S_0| < +\infty$ and a set of non-zero values $\{ \gamma_{ij} \in \mathbb{R} | (i,j) \in S_0 \}$, such that 
\begin{equation}
\label{eq:linear_dependence_division}
\sum_{ (i, j) \in S_0} \gamma_{ij}v_{j}(\z, \y)u_i(\y, \x) = 0, \\
\quad \forall (\x, \y, \z) \in \mathcal{X}\times \mathbb{R}^{d_y} \times \mathcal{Z}.
\end{equation}
Note that the choices of $S_0$ and $\gamma_{ij}$ are independent of any $(\x, \y, \z)$ values, but might be dependent on $\mathcal{Z}$. From Lemma \ref{lemma:extension_b4} we know that the family $\mathcal{V}_J$ as defined in (b), satisfies (a4) from Lemma \ref{lemma:linear_independence_two_nonlinear_gaussians}. Furthermore, it also satisfies (a1) as it is defined as a product of $n$ functions that satisfy (b1), with $1\leq n \leq M$. Therefore, we follow the same arguments as in Lemma \ref{lemma:linear_independence_two_nonlinear_gaussians}, which show that the index set $S_0$ contains at least 2 different indices $(i, j)$ and $(i, j')$ with $j
\neq j'$. Moreover, we define $J_0 = \{j \in A | \exists (i, j) \in S_0 \}$ the set of all possible $j$ indices that appear in $S_0$, where $|J_0| < +\infty$ and the following set $\mathcal{Y}_0 := \{\bm{\beta} \in \mathbb{R}^{d_y} | \sum_{ i: (i, j) \in S_0} \gamma_{ij} u_{i}(\y = \bm{\beta}, \x) = 0, \forall \x \in \mathcal{X} \}$ can only have zero measure in $\mathbb{R}^{d_y}$ from (a3). Again following Lemma \ref{lemma:linear_independence_two_nonlinear_gaussians}, we have a non-zero measure set $\mathcal{Y}_1 := \mathcal{Y} \backslash \mathcal{Y}_0 \subset \mathcal{Y}$ where we choose $\mathcal{Y}$ from the non-zero measure set in (b4). Then, we have for each $\bm{\beta} \in \mathcal{Y}_1$, there exists $\x \in \mathcal{X}$ such that $\sum_{ i: (i, j) \in S_0} \gamma_{ij} u_{i}(\y = \bm{\beta}, \x) \neq 0$ for at least two $j$ indices in $J_0$. This implies for each $\bm{\beta} \in \mathcal{Y}_1$, $\{v_j(\z, \y = \bm{\beta}) | j \in J_0 \}$ contains linearly dependent functions on $\z \in \mathcal{Z}$.

From (b4-b6), Lemma \ref{lemma:extension_b5} shows that the product of $n$ functions ($1\leq n \leq M$) that compose the family $\mathcal{V}_J$ implies linear dependent functions of at least one of the components. That is, for some $m \leq n$,  $\{w_{k_m}(\z_{m}, \dots, \z_1,\y_{M:m}=\bm{\beta}_{M:m}), k_m:(k_n,\dots,k_1)\in J_0\}$, with $\y_{M:m}:=(\y_M,\dots,\y_m)$, $\bm{\beta}_{M:m}:=(\bm{\beta}_M,\dots,\bm{\beta}_m)$, and $J_0 := (K_0^{(n)} \times \dots \times K_0^{(1)})$, contains linear dependent functions on $(\z_m,\dots,\z_1) \in \pi(\mathcal{Z}, \{m,\dots, 1\})$. Under assumption (b5), we can split the index set $K_0^{(m)}$ into subsets indexed by $l \in L^{(m)}(\bm{\beta}_{M:m})$ as follows, such that within each index subset $K_l^{(m)}(\bm{\beta}_{M:m})$ the functions with the corresponding indices are equal:
\begin{equation}
\label{eq:j_index_split_def}
\begin{aligned}
K_0^{(m)}& = \cup_{l \in L^{(m)}(\bm{\beta}_{M:m})} K_l^{(m)}(\bm{\beta}_{M:m}), \quad K_l^{(m)}(\bm{\beta}_{M:m}) \cap K_{l'}^{(m)}(\bm{\beta}_{M:m}) = \emptyset, \forall l \neq l' \in L^{(m)}(\bm{\beta}_{M:m}), \\
& k_m \neq k_m' \in  K_l^{(m)}(\bm{\beta}_{M:m}) \quad \Leftrightarrow \quad w_{k_m}(\z_{m}, \dots, \z_1,\y_{M:m}=\bm{\beta}_{M:m}) = w_{k'_m}(\z_{m}, \dots, \z_1,\y_{M:m}=\bm{\beta}_{M:m}), \\ & \qquad \forall (\z_m,\dots,\z_1) \in \pi(\mathcal{Z}, \{m,\dots, 1\}).
\end{aligned}
\end{equation}

Then we can rewrite Eq.~(\ref{eq:linear_dependence_division}) for any $\bm{\beta} \in \mathcal{Y}_1$ as
\begin{multline}
\sum_{l \in L^{(m)}(\bm{\beta}_{M:m})} \Bigg(\sum_{k_m \in K^{(m)}_l(\bm{\beta}_{M:m})}\sum_{ (i, k_n, \dots, k_{m+1}, k_{m-1}, \dots, k_1): (i, k_n, \dots, k_1) \in S_0} \gamma_{i,k_n,\dots,k_1} u_{i}(\y = \bm{\beta}, \x) \\ w_{k_n}(\z_n,\dots,\z_1, \y_M=\bm{\beta}_M)\dots  w_{k_1}(\z_1, \y_M=\bm{\beta}_M,\dots, \y_1=\bm{\beta}_1) \Bigg) = 0, \quad \forall (\x, \z_n, \dots, \z_1) \in \mathcal{X} \times \mathcal{Z}.  
\end{multline}
Note that we use $\gamma_{i,k_n,\dots,k_1} = \gamma_{ij}$ for any $(i,j)\in S_0$ and $(k_n,\dots,k_1)\in J_0$, such that $j = (k_n,\dots,k_1)$. Recall from Eq.~(\ref{eq:j_index_split_def}) that $w_{k_m}(\z_{m}, \dots, \z_1,\y_{M:m}=\bm{\beta}_{M:m})$ and $w_{k'_m}(\z_{m}, \dots, \z_1,\y_{M:m}=\bm{\beta}_{M:m})$ are the same functions on $(\z_m, \dots, \z_1) \in \pi(\mathcal{Z}, \{m,\dots, 1\})$ iff.~$k_m \neq k_m'$ are in the same index set $K_l^{(1)}(\bm{\beta})$. This means if Eq.~(\ref{eq:linear_dependence_division}) holds, then for any $\bm{\beta} \in \mathcal{Y}_1$, under assumptions (b1) and (b5),
\begin{multline}
\label{eq:zero_constraint_given_beta}
\sum_{k_m \in K^{(m)}_l(\bm{\beta}_{M:m})}\sum_{ (i, k_n, \dots, k_{m+1}, k_{m-1}, \dots, k_1): (i, k_n, \dots, k_1) \in S_0} \gamma_{i,k_n,\dots,k_1} u_{i}(\y = \bm{\beta}, \x) w_{k_n}(\z_n,\dots,\z_1, \y_M=\bm{\beta}_M)\\ \dots w_{k_{m+1}}(\z_{m+1},\dots,\z_1, \y_{M:m+1}=\bm{\beta}_{M:m+1})w_{k_{m-1}}(\z_{m-1},\dots,\z_1, \y_{M:m-1}=\bm{\beta}_{M:m-1}) \\ \dots w_{k_1}(\z_1, \y_{M:1}=\bm{\beta}_{M:1}) = 0 ,\quad \forall (\x, \z_n, \dots, \z_1) \in \mathcal{X} \times \mathcal{Z},\text{  } l \in L^{(m)}(\bm{\beta}_{M:m}).
\end{multline}

Define $C^{(m)}(\bm{\beta}_{M:m}) = \min_l |K^{(m)}_l(\bm{\beta}_{M:m})|$ the minimum cardinality count for $k_m$ indices in the $K^{(m)}_l(\bm{\beta}_{M:m})$ subsets. Choose $\bm{\beta}^*_{M:m} \in \arg\min_{\bm{\beta}_{M:m} \in \pi(\mathcal{Y}_1, \{M,\dots, m\})} C^{(m)}(\bm{\beta}_{M:m})$:
\begin{itemize}
    \item[1.] We have $C^{(m)}(\bm{\beta}^*_{M:m}) < |K^{(m)}_0|$ and $|L^{(m)}(\bm{\beta}^*_{M:m})| \geq 2$. Otherwise for all $k_m \neq k_m' \in K^{(m)}_0$ we have $w_{k_m}(\z_m,\dots,\z_1, \y_{M:m}=\bm{\beta}_{M:m}) = w_{k_m'}(\z_m,\dots,\z_1, \y_{M:m}=\bm{\beta}_{M:m})$ for all $(\z_m, \dots, \z_1) \in \pi(\mathcal{Z}, \{m,\dots, 1\})$ and $\bm{\beta}_{M:m} \in \pi(\mathcal{Y}_1, \{M,\dots,m\})$, so that they are linearly dependent on $(\z_m, \dots, \z_1, \y_{M:m}) \in \pi(\mathcal{Z}, \{m,\dots, 1\}) \times \pi(\mathcal{Y}_1, \{M,\dots, m\})$, a contradiction to assumption (b4).% by setting $\mathcal{Z}\times\mathcal{Y}' = \pi(\mathcal{Z},\{m,\dots, 1\})\times \pi(\mathcal{Y}_1, \{M,\dots,m\})$.
    \item[2.] Now assume $|K_1^{(m)}(\bm{\beta}^*_{M:m})| = C^{(m)}(\bm{\beta}^*_{M:m})$ w.l.o.g.. From assumption (b5), we know that for any $k_m \in K_1^{(m)}(\bm{\beta}^*_{M:m})$ and $k_m' \in K_0^{(m)} \backslash K_{1}^{(m)}(\bm{\beta}^*_{M:m})$, $w_{k_m}(\z_m, \dots, \z_1, \y_{M:m}=\bm{\beta}_{M:m}) = w_{k_m'}(\z_m, \dots, \z_1, \y_{M:m}=\bm{\beta}_{M:m})$ only on zero measure subset of $\pi(\mathcal{Z}, \{m, \dots, 1\})$ at most. Then as $|K_0^{(m)}| < +\infty$ and $\pi(\mathcal{Z}, \{m,\dots,1\}) \subset (\times^{m} \mathbb{R}^{d})$ has non-zero measure, there exist $(\z^{(0)}_m, \dots, \z^{(0)}_1) \in \pi(\mathcal{Z}, \{m,\dots,1\})$ and $\delta > 0$ such that 
    \begin{multline}
    |w_{k_m}(\z_{m:1}=\z_{m:1}^{(0)}, \y_{M:m}=\bm{\beta}^*_{M:m}) - w_{k_m'}(\z_{m:1}=\z_{m:1}^{(0)}, \y_{M:m}=\bm{\beta}^*_{M:m})| \geq \delta, \\ \forall k_m \in K_{1}^{(m)}(\bm{\beta}^*_{M:m}), \forall k_m' \in K_0^{(m)} \backslash K_{1}^{(m)}(\bm{\beta}^*_{M:m}).
    \end{multline}
    Under assumption (b6), there exists $\epsilon(k_m) > 0$ such that we can construct an $\epsilon$-ball $B_{\epsilon(k_m)}(\bm{\beta}^*_{M:m})$ using $\ell_2$-norm, such that 
    \begin{multline}
    |w_{k_m}(\z_{m:1}=\z_{m:1}^{(0)}, \y_{M:m}=\bm{\beta}^*_{M:m}) - w_{k_m}(\z_{m:1}=\z_{m:1}^{(0)}, \y_{M:m}=\bm{\beta}_{M:m})| \leq \delta/3, \text{ 
 } \forall \bm{\beta}_{M:m} \in B_{\epsilon(k_m)}(\bm{\beta}^*_{M:m}).
    \end{multline}
    Choosing a suitable $0 < \epsilon \leq \min_{k_m \in K_0^{(m)}} \epsilon(k_m)$ (note that $\min_{k_m \in K_0^{(m)}} \epsilon(k_m) > 0$ as $|K_0^{(m)}| < +\infty$) and constructing an $\ell_2$-norm-based $\epsilon$-ball $B_{\epsilon}(\bm{\beta}^*_{M:m}) \subset \pi(\mathcal{Y}_1,\{M,\dots,m\})$, we have for all $k_m \in K_{1}^{(m)}(\bm{\beta}^*_{M:m}), k_m' \in K_0^{(m)} \backslash K_{1}^{(m)}(\bm{\beta}^*_{M:m})$, $k_m' \notin K_1^{(m)}(\bm{\beta}_{M:m})$ for all $\bm{\beta}_{M:m} \in B_{\epsilon}(\bm{\beta}^*_{M:m})$ due to
    \begin{multline}
    |w_{k_m}(\z_{m:1}=\z_{m:1}^{(0)}, \y_{M:m}=\bm{\beta}_{M:m}) - w_{k_m'}(\z_{m:1}=\z_{m:1}^{(0)}, \y_{M:m}=\bm{\beta}_{M:m})| \geq \delta/3, \text{ 
 } \forall \bm{\beta}_{M:m} \in B_{\epsilon}(\bm{\beta}^*_{M:m}).
    \end{multline}
    So this means for the split $\{ K_l^{(m)}(\bm{\beta}_{M:m}) \}$ of any $\bm{\beta} \in B_{\epsilon}(\bm{\beta}^*_{M:m})$, we have $ K_1^{(m)}(\bm{\beta}_{M:m}) \subset K_1^{(m)}(\bm{\beta}^*_{M:m})$ 
    and therefore $ |K_1^{(m)}(\bm{\beta}_{M:m})| \leq |K_1^{(m)}(\bm{\beta}^*_{M:m})|$. Now by definition of $\bm{\beta}^*_{M:m} \in \arg\min_{\bm{\beta}_{M:m} \in \pi(\mathcal{Y},\{M,\dots,m\})} C^{(m)}(\bm{\beta}_{M:m})$ and $|K_1^{(m)}(\bm{\beta}^*_{M:m})| = C^{(m)}(\bm{\beta}^*_{M:m})$, we have $K_1^{(m)}(\bm{\beta}_{M:m}) = K_1^{(m)}(\bm{\beta}^*_{M:m})$ for all $\bm{\beta}_{M:m} \in B_{\epsilon}(\bm{\beta}^*_{M:m})$. 
    \item[3.] One can show that $|K_1^{(m)}(\bm{\beta}^*_{M:m})| = 1$, otherwise by definition of the split (Eq.~(\ref{eq:j_index_split_def})) and the above point, there exists $k_m \neq k_m' \in K_1^{(m)}(\bm{\beta}^*_{M:m})$ such that $w_{k_m}(\z_{m}, \dots, \z_1,\y_{M:m}=\bm{\beta}_{M:m})=w_{k'_m}(\z_{m}, \dots, \z_1,\y_{M:m}=\bm{\beta}_{M:m})$ for all $(\z_m, \dots, \z_1) \in \pi(\mathcal{Z},\{m,\dots,1\})$ and $\bm{\beta}_{M:m} \in B_{\epsilon}(\bm{\beta}^*_{M:m})$, again a contradiction to assumption (b4).% by setting $\mathcal{Z}\times\mathcal{Y}' = \pi(\mathcal{Z},\{m,\dots, 1\})\times B_{\epsilon}(\bm{\beta}^*_{M:m})$. 
 \end{itemize}   
    Given the above 3 points and assuming $k_m \in K_1^{(m)}(\bm{\beta}^*_{M:m})$ is the only index in the subset, then the fact proved in the above point that $K_1^{(m)}(\bm{\beta}_{M:m}) = K_1^{(m)}(\bm{\beta}^*_{M:m})$ means
\begin{multline}
\sum_{ (i, k_n, \dots, k_{m+1}, k_{m-1}, \dots, k_1): (i, k_n, \dots, k_1) \in S_0} \gamma_{i,k_n,\dots,k_1} u_{i}(\y = \bm{\beta}, \x) w_{k_n}(\z_n,\dots,\z_1, \y_M=\bm{\beta}_M)\\ \dots w_{k_{m+1}}(\z_{m+1},\dots,\z_1, \y_{M:m+1}=\bm{\beta}_{M:m+1})w_{k_{m-1}}(\z_{m-1},\dots,\z_1, \y_{M:m-1}=\bm{\beta}_{M:m-1}) \\ \dots w_{k_1}(\z_1, \y_{M:1}=\bm{\beta}_{M:1}) = 0 ,\quad \forall (\x, \z_n, \dots, \z_1) \in \mathcal{X} \times \mathcal{Z}, \forall \bm{\beta}_{M:m} \in B_{\epsilon}(\bm{\beta}^*_{M:m});
\end{multline}
where the term $w_{k_{m}}(\z_{m},\dots,\z_1, \y_{M:m}=\bm{\beta}_{M:m})$ can be omitted from (b1). Now, by setting $\mathcal{Y} = (B_{\epsilon}(\bm{\beta}^*_{M:m})\times \pi(\mathcal{Y}_1, \{m-1,\dots, 1\}))\setminus \mathcal{Y}_0$, this implies linear dependence on the family $\{w_{k_n}(\z_n,\dots,\z_1, \y_M=\bm{\beta}_M) \dots w_{k_{m+1}}(\z_{m+1},\dots,\z_1, \y_{M:m+1}=\bm{\beta}_{M:m+1})w_{k_{m-1}}(\z_{m-1},\dots,\z_1, \y_{M:m-1}=\bm{\beta}_{M:m-1}) \dots w_{k_1}(\z_1, \y_{M:1}=\bm{\beta}_{M:1})| (k_n,\dots,k_2)\in(K_0^{(n)}\times \dots \times K^{(m+1)} \times K^{(m-1)}\dots\times K_0^{(1)})\}$, and from Lemma \ref{lemma:extension_b5}, we know there is at least one $m'\leq n$ with $m'\neq m$ such that $\{w_{k_{m'}}(\z_{m'}, \dots, \z_1,\y_{M:m'}=\bm{\beta}_{M:m'}), k_m':(k_n,\dots,k_1)\in J_0\}$ contains linear independent functions under finite mixtures. 

The previous argument can be repeatedly applied to show linear dependence on every product element in the $\mathcal{W}_K$ family. Finally we have
\begin{equation}
\sum_{ i: (i, k_n, \dots, k_1) \in S_0} \gamma_{i,k_n,\dots,k_1} u_{i}(\y = \bm{\beta}, \x) = 0, \quad \forall (\x, \y) \in \mathcal{X} \times B_{\epsilon}(\bm{\beta}^*), 
\end{equation}
where we set $\bm{\beta}^* := \bm{\beta}_{M:1}$. For any $m$, and $\bm{\beta}_{M:m} \in B_{\epsilon}(\bm{\beta}^*_{M:m})$, we have  $|K_1^{(m)}(\bm{\beta}_{M:m})| \leq |K_1^{(m)}(\bm{\beta}^*_{M:m})|$ (shown in step 2 above). Therefore by definition of $\bm{\beta}^*_{M:m} \in \arg\min_{\bm{\beta}_{M:m} \in \pi(\mathcal{Y},\{M,\dots,m\})} C^{(m)}(\bm{\beta}_{M:m})$, we can obtain alignment of $\bm{\beta}^*_{M:m}$ w.r.t. any $\bm{\beta}^*_{M:m'}$ where $m'\neq m$. I.e. for any $m\neq m'$, with $1 \leq m < m'\leq n$ wlog., we have two $\bm{\beta}^{*(m)}_{M:m}$, $\bm{\beta}^{*(m')}_{M:m'}$, such that $\bm{\beta}^{*(m)}_{M:m'} = \bm{\beta}^{*(m')}_{M:m'}$. This shows a contradiction to (a3) by fixing $
\mathcal{Y}= B_{\epsilon}(\bm{\beta}^*)$. 
\end{proof}

\subsubsection{Linear independence of non-parametric joint distribution family}

Below we present linear independence for non-parametric joint distribution families.

\begin{theorem}
\label{thm:linear_independence_joint_non_parametric}
Define the following joint distribution family
\begin{multline}
\Bigg\{ p_{a_{1:T}}(\x_{1:T}) = p_{a_{1:M}}(\x_{1:M})\prod_{t=M+1}^T p_{a_t}(\x_t | \x_{t-1},\dots, \x_{t-M}), \quad p_{a_{1:M}} \in \Pi_A^M, p_{a_t} \in \mathcal{P}_A^M, t = M+1, ..., T \Bigg\},
\end{multline}
and assume $\Pi_A^M$ and $\mathcal{P}_A^M$ satisfy the following assumptions,
\begin{itemize}
    \item[(c1)] $\Pi^M_A$ satisfies (a1), (a3), and
    \item[(c2)] $\mathcal{P}_A^M$ satisfies (b4-b6).
\end{itemize}
Then the following statement holds: For any $T > M$ such that $n=M$ if $T \equiv 0 \mod M$, or $n=(T \mod M)$, otherwise; and any subset $\underbrace{\mathcal{X} \times \dots \times \mathcal{X}}_{n\text{ times}} \subset \mathbb{R}^{nd}$, the joint distribution family contains linearly independent distributions under finite mixtures for $(\x_{1:T-n}, \x_{T-n+1:T}) \in \mathbb{R}^{(T-n)d} \times \underbrace{\mathcal{X} \times \dots \times \mathcal{X}}_{n\text{ times}}$.
\end{theorem}
\begin{proof}
We first prove the statement for any $T>M$ such that $T \equiv 0 \mod M$ by induction as follows.

\underline{$T = 2M$}: The result can be proved using Lemma \ref{lemma:absolutely_magical_lemma} by setting in the proof, $u_i(\y = \x_{1:M}, \x = \x_0) = \pi_{a_{1:M}}(\x_{1:M}), i = a_{1:M}$ and $v_j(\z = \x_{M+1:2M}, \y = \x_{1:M}) = \prod_{m=1}^M p_{a_{M+m}}(\x_{M+m} | \x_{m:M-1+m}), j=a_{M+1:2M}$. We observe that the Lemma holds using assumptions (c1-c2) directly as $v_j(\z = \x_{M+1:2M}, \y = \x_{1:M})$ is a product of $M$ functions that satisfies (b) in Lemma \ref{lemma:absolutely_magical_lemma}.

\underline{$T = aM, a\in\mathbb{Z}^{+}, a> 2$}: Assume the statement holds for the joint distribution family when $T = \tau - M$.
Note that we can write $p_{a_{1:\tau}}(\x_{1:\tau})$ as
\begin{equation}
p_{a_{1:\tau}}(\x_{1:\tau}) = p_{a_{1:\tau-M}}(\x_{1:\tau-M}) \prod_{m=1}^M p_{a_{\tau-M+m}}(\x_{\tau-M+m} | \x_{\tau-2M+m:\tau-M-1+m}).
\end{equation}
Then the statement for $T = \tau$ can be proved using Lemma \ref{lemma:absolutely_magical_lemma} by setting
$u_i(\y = \x_{\tau - 2M+1 : \tau - M}, \x = \x_{1:\tau-2M}) = p_{a_{1:\tau-M}}(\x_{1:\tau-M}), i = a_{1:\tau-M}$, and $v_j(\z = \x_{\tau-M+1:\tau}, \y = \x_{\tau - 2M +1 : \tau - M}) = \prod_{m=1}^M p_{a_{\tau-M+m}}(\x_{\tau-M+m} | \x_{\tau-2M+m:\tau-M-1+m}), j=a_{\tau-M+1:\tau}$. Note that the family spanned with $p_{a_{1:\tau-M}}(\x_{1:\tau-M}), i = a_{1:\tau-M}$ satisfies (a1) from (c1), and (a3) from the induction hypothesis. For the family $\prod_{m=1}^M p_{a_{\tau-M+m}}(\x_{\tau-M+m} | \x_{\tau-2M+m:\tau-M-1+m}), j=a_{\tau-M+1:\tau}$, we have a product of $M$ functions where (b1) and (b4-b6) are satisfied from (c2), which imply (b) in Lemma \ref{lemma:absolutely_magical_lemma}.

Given the above result we proceed to prove the case where $T \not\equiv 0 (\mod M)$.

\underline{$T \not\equiv 0 \mod M$}: Let $n$ be the remainder of $T/M$: $T\equiv n \mod M$. From the previous result we know the statement holds for $T - n$, i.e. when $T$ is a multiple of $M$. We can write $p_{a_{1:T}}(\x_{1:T})$ as follows
\begin{equation}
p_{a_{1:T}}(\x_{1:T}) = p_{a_{1:T-n}}(\x_{1:T-n}) \prod_{m=1}^n p_{a_{T-n+m}}(\x_{T-n+m} | \x_{T-n+m-M:T-n+m-1}).
\end{equation}
We can again prove the statement with Lemma \ref{lemma:absolutely_magical_lemma} by setting $u_i(\y = \x_{T - n-M+1 : T - n}, \x = \x_{1:T-n-M}) = p_{a_{1:T-n}}(\x_{1:T-n}), i = a_{1:T-n}$, and $v_j(\z = \x_{T-n+1:T}, \y = \x_{T - n-M +1 : T - n}) = \prod_{m=1}^n p_{a_{T-n+m}}(\x_{T-n+m} | \x_{T-n+m-M:T-n+m-1}), j=a_{\tau-n+1:\tau}$. The family spanned with $p_{a_{1:T-n}}(\x_{1:T-n}), i = a_{1:T-n}$ satisfies (a1) from (c1). (a3) is satisfied as follows:
\begin{itemize}
    \item For $M < T<2M$, we have $T-n=M$ and (a3) holds directly from (c1).
    \item For $T > 2M$, we have $T-n \equiv 0 \mod M$ and (a3) holds from the previous induction technique.    
\end{itemize}
The family $\prod_{m=1}^n p_{a_{T-n+m}}(\x_{T-n+m} | \x_{T-n+m-M:T-n+m-1})$ is a product of $n$ functions which satisfy (b1) and (b4-b6) from (c2). This implies (b) in Lemma \ref{lemma:absolutely_magical_lemma}.

\end{proof}

%The requirement $T=a\cdot M, a\in\mathbb{Z}^+$ can be solved by providing base cases results for $T=M+1, \cdot, T=2M-1$. However, we leave this for future work, as such cases are dependent on the initial distribution.

\subsection{Linear independence under assumptions (i) and (ii)}\label{app:parametric_assumptions}

Below we explore the parametric conditions in which linear independence of the joint distribution (and thus, identifiability) can be achieved. We start from the following result from \citet{balsells-rodas2024on}.

\begin{proposition}
\label{prop:gaussian_linear_independence}
Functions in $\mathcal{G}_A$ are linearly independent on variables $(\x_t, \dots, \x_{t-M})$ if the unique indexing assumption (Eq.~(\ref{eq:unique_indexing_conditional_gaussian})) holds.
\end{proposition}

The transition and initial Gaussian distribution families defined in Eqs. (\ref{eq:gaussian_transition_family}) and (\ref{eq:gaussian_initial_family}) respectively align with assumptions (a) and (b) as follows. 

\begin{proposition}
\label{prop:conditional_gaussian_properties}
The conditional Gaussian distribution family $\mathcal{G}_{\A}$ (Eq.~(\ref{eq:gaussian_transition_family})), under the unique indexing assumption (Eq.~(\ref{eq:unique_indexing_conditional_gaussian})), satisfies assumptions (b1) and (b5) in Lemmas \ref{lemma:extension_b4}, \ref{lemma:extension_b5}, and \ref{lemma:absolutely_magical_lemma}, if we define $\mathcal{V}_J := \mathcal{G}_{\A}, \z: = \x_t$ and $(\y_M,\dots,\y_1) := (\x_{t-1},\dots,\x_{t-M})$.
\end{proposition}
\begin{proposition}
\label{prop:marginal_gaussian_properties}
The initial Gaussian distribution family $\mathcal{I}_{\A}$ (Eq.~(\ref{eq:gaussian_initial_family}), under the unique indexing assumption (Eq.~(\ref{eq:unique_indexing_marginal_gaussian}), satisfies assumptions (a1), (a3) in Lemma \ref{lemma:absolutely_magical_lemma}, if we define $\mathcal{U}_I := \mathcal{I}_{\A}, \y: = \x_1$ and $\x = \mathcal{X} = \emptyset$.
\end{proposition}
To see why (b5) holds under nonlinear Gaussians, we note from Prop. \ref{prop:gaussian_linear_independence} that linear dependence occurs only if the unique indexing assumption does not hold. Therefore, we can fix a subset of the $(\x_{t-M}, \dots, \x_{t-1})$ variables, such that the resulting mean and covariance functions violate the unique indexing assumption, which would imply linear dependence on the resulting function family. To verify (b4), we require the following zero-measure intersection of moments result.
\begin{proposition}\label{prop:intersection_mom}
Assume Gaussian family transitions $\mathcal{G}_A$ under unique indexing defined by Eq. (\ref{eq:unique_indexing_conditional_gaussian}), with zero-measure intersection of moments:  For $A$ such that $|A| < +\infty$, and any $a,a'\in A$ with $a \neq a'$,  the set $\mathcal{X}_{a,a'}:=\{ (\x_{t-1},\dots, \x_{t-M}) \in \mathbb{R}^{dM} | \bm{m}(\x_{t-1},\dots,\x_{t-M}, a) = \bm{m}(\x_{t-1},\dots,\x_{t-M}, a'), \bm{\Sigma}(\x_{t-1},\dots,\x_{t-M}, a) = \bm{\Sigma}(\x_{t-1},\dots,\x_{t-M}, a') \}$ has zero measure. Then (b4) holds if $V_J:=G_A$, $\z:=\x_t$, and $(\y_1,\dots,\y_M):=(\x_{t-1},\dots,\x_{t-M})$.
\end{proposition}
\begin{proof}
    Define $\mathcal{V}_J := \mathcal{G}_{A}$, $\z:=\x_t$, and $(\y_1,\dots,\y_M):=(\x_{t-1},\dots,\x_{t-M})$ from (b4). We set $\mathcal{Y}:= \R^{dM}\setminus \bigcup_{a\neq a' \in A}\mathcal{X}_{a,a'}$, where we have $\mu(\mathcal{Y}) = \mu(\times_{m=1}^M \R^d)$. Thus, we are guaranteed unique indexing for any $\mathcal{Y}'\subset\mathcal{Y}$, and from Gaussian identifiability $\mathcal{G}_A$ contains linear independent functions under finite mixtures on $(\z, \y_1,\dots,\y_M)\in (\mathcal{Z}\times\mathcal{Y}')\subset (\R^d \times\mathcal{Y}$).
\end{proof}

A similar assumption is presented in \citet{balsells-rodas2024on}, where it is assumed within a certain non-zero measure set. Now we restrict it to hold in a full measure set of $\R^{dM}$. Finally, we present linear independence of the joint distribution in the parametric case.

\begin{theorem}
\label{thm:linear_independence_nonlinear_gaussian}
Define the following joint distribution family under the non-linear Gaussian model
\begin{equation}\label{eq:joint_distrib_fam}
\mathcal{P}_{\A}^T = \Bigg\{ p_{a_{1:T}}(\x_{1:T}) = p_{a_{1:M}}(\x_{1:M})\prod_{t=M+1}^T p_{a_t}(\x_t | \x_{t-1},\dots, \x_{t-M}), \quad p_{a_{1:M}} \in \mathcal{I}^M_{\A}, \quad p_{a_t} \in \mathcal{G}^M_{\A},\quad t = 2, ..., T \Bigg\},
\end{equation}
with $\mathcal{G}^M_{\A}$, $\mathcal{I}^M_{\A}$ defined by Eqs. (\ref{eq:gaussian_transition_family}), (\ref{eq:gaussian_initial_family}) respectively. Assume:
\begin{itemize}
    \item[(d1)] Unique indexing for $\mathcal{G}_{\A}$ and $\mathcal{I}_{\A}$: Eqs.~(\ref{eq:unique_indexing_conditional_gaussian}), ~(\ref{eq:unique_indexing_marginal_gaussian}) hold;
    \item[(d2)] The functions in $\mathcal{G}_{\A}$ are continuous with respect to $(\x_{t-1},\dots,\x_M) \in \mathbb{R}^dM$;
    \item[(d3)] Zero-measure intersection of moments:  For $A$ such that $|A| < +\infty$, and any $a,a'\in A$ with $a \neq a'$,  the set $\mathcal{X}_{a,a'}:=\{ (\x_{t-1},\dots, \x_{t-M}) \in \mathbb{R}^{dM} | \bm{m}(\x_{t-1},\dots,\x_{t-M}, a) = \bm{m}(\x_{t-1},\dots,\x_{t-M}, a'), \bm{\Sigma}(\x_{t-1},\dots,\x_{t-M}, a) = \bm{\Sigma}(\x_{t-1},\dots,\x_{t-M}, a') \}$ has zero measure.
\end{itemize}
Then, the joint distribution family contains linearly independent distributions under finite mixtures for $(\x_{1:T-n}, \x_{T-n+1}, \dots,\x_T) \in \mathbb{R}^{(T-n)d} \times (\times_{m=1}^n \mathbb{R}^d)$, where $n=M$ if $T\mod M =0$, or $n=T\mod n$ otherwise.
\end{theorem}
\begin{proof}
Note that (d2) and (b6) are equivalent. Assumptions (a1), (a3), (b1), and (b5) are satisfied due to Propositions \ref{prop:conditional_gaussian_properties}, \ref{prop:marginal_gaussian_properties}. Assumption (b4) holds due to assumption (d3) via Prop \ref{prop:intersection_mom}. Then, the statement holds by Theorem \ref{thm:linear_independence_joint_non_parametric}.
\end{proof}

\subsection{Conclusion}

Below we conclude the proof of Theorem \ref{thm:identifiability_main}
\begin{proof}
    Linear independence under finite mixtures of the joint distribution family $\jointfamily$ holds from Theorem \ref{thm:linear_independence_nonlinear_gaussian}, and the statement is proved by \ref{thm:identifiability_msm}. Below we clarify the alignment from assumptions (i-ii) to (d1-d3). (d1) and (i) are equivalent. (d2) is satisfied as analytic functions are $\mathcal{C}^{\infty}$, and therefore continous. From \citet{mityagin2015zero}, we know the zero set of an analytic function is zero-measured. Therefore, intersections between on means or covariances for any $a,a'\in\mathcal{A}$, with $a\neq a'$ are also zero-measured, which implies (d3).
\end{proof}

\section{Proof of Corollary \ref{cor:causality_identifiability}}\label{app:causality_proof}

\begin{proof}
    Assume two MSMs $p(\x_{1:T}),\tilde{p}(\x_{1:T})$ with corresponding regime-dependent graphs $\mathbf{G}_{1:K}$, $\tilde{\mathbf{G}}_{1:\tilde{K}}$. Given assumptions (i-ii) and Theorem \ref{thm:identifiability_main}, we have identifiability up to permutations from Def. \ref{def:identifiability}. We focus on condition 4 where for each $1 \leq i \leq C$, there is some $1\leq j \leq \tilde{C}$ such that:
    \begin{equation}
        p_{b^i_t}(\x_t | \x_{t-1:t-M})=p_{\tilde{b}^j_t}(\x_t | \x_{t-1:t-M}), \quad\forall(\x_t,\dots, \x_{t-M}) \in \R^{d(M+1)}.
    \end{equation}
    Then, at each time step $t$ there exists a permutation $\sigma$ such that for each $k\in \{1,\dots,K\}$, $p(\x_t | \x_{t-1:t-M}, s_t=k)=\tilde{p}(\x_t | \x_{t-1:t-M}, s_t=\sigma(k)), \quad\forall(\x_t,\dots, \x_{t-M}) \in \R^{d(M+1)}$. We know from condition 2 that $\sigma$ is constant for $t > M$. Since we know the distributions are Gaussian, from Gaussian identifiability \citep{yakowitz1968identifiability} we have
    \begin{equation}
        \bm{m}(\x_t,\dots, \x_{t-M}, k) = \tilde{\bm{m}}(\x_t,\dots, \x_{t-M}, \sigma(k)) \quad \bm{\Sigma}(k)=\tilde{\bm{\Sigma}}(\sigma(k)), \quad k\in\{1,\dots,K\}.
    \end{equation}
    Now, given that the functions $\bm{m}(\cdot,k), k\in\B$ are analytic, the Jacobian will preserve similar permutation equivalences. Wlog. we fix $\tau$, where $1 \leq \tau \leq M$ and compute the Jacobian w.r.t. $\x_{t-\tau}$, denoted as $J_{\bm{m}(\cdot,k), \tau}$. We will have the following equivalence
    \begin{multline}
        J_{\bm{m}(\cdot,k), \tau} = \begin{pmatrix}
            \frac{\partial \bm{m}^{(1)}(\x_{t-1}, \dots, \x_{t-M},k) }{\partial \x^{(1)}_{t-\tau}} & \dots & \frac{\partial \bm{m}^{(1)}(\x_{t-1}, \dots, \x_{t-M},k) }{\partial \x^{(d)}_{t-\tau}} \\
            \vdots \\
            \frac{\partial \bm{m}^{(d)}(\x_{t-1}, \dots, \x_{t-M},k) }{\partial \x^{(1)}_{t-\tau}} & \dots & \frac{\partial \bm{m}^{(d)}(\x_{t-1}, \dots, \x_{t-M},k) }{\partial \x^{(d)}_{t-\tau}}
        \end{pmatrix} = \\ \begin{pmatrix}
            \frac{\partial \tilde{\bm{m}}^{(1)}(\x_{t-1}, \dots, \x_{t-M},\sigma(k)) }{\partial \x^{(1)}_{t-\tau}} & \dots & \frac{\partial \tilde{\bm{m}}^{(1)}(\x_{t-1}, \dots, \x_{t-M},\sigma(k)) }{\partial \x^{(d)}_{t-\tau}} \\
            \vdots \\
            \frac{\partial \tilde{\bm{m}}^{(d)}(\x_{t-1}, \dots, \x_{t-M},\sigma(k)) }{\partial \x^{(1)}_{t-\tau}} & \dots & \frac{\partial \tilde{\bm{m}}^{(d)}(\x_{t-1}, \dots, \x_{t-M},\sigma(k)) }{\partial \x^{(d)}_{t-\tau}}
        \end{pmatrix}  = J_{\bm{m}(\cdot,\sigma(k)), \tau}
    \end{multline}
    where $\bm{m}^{(i)}(\x_{t-1}, \dots, \x_{t-M},k)$ denotes the $i$-th dimension of $\bm{m}(\x_{t-1}, \dots, \x_{t-M},k)$. Given the minimality assumption (iii), for each $k\in \{1,\dots,K\}$, and $1 \leq \tau \leq M$, we have $\x^{(i)} \in \mathbf{Pa}^{(j)}_{k}(\tau)$ if $\frac{\partial \bm{m}^{(i)}(\x_{t-1}, \dots, \x_{t-M},k) }{\partial \x^{(j)}_{t-\tau}} \neq 0$. This implies we can retrieve $\mathbf{G}_{1:K}$ and $\tilde{\mathbf{G}}_{1:\tilde{K}}$ using the Jacobians; and from the above equation we have $\mathbf{G}_{k}=\tilde{\mathbf{G}}_{\sigma(k)}$, for each $k\in \{1,\dots,K\}$.
\end{proof}

\section{Experiment Details}\label{app:experiment_details}

\subsection{Data Generation}

As mentioned earlier, we sample $N=10000$ sequences of length $T=200$. We assume a first-order Markov chain with $K$ states. The initial distribution assigns equal probability across states and the trainsition matrix maintains the same state with $90\%$ probability, or transitions to the next state with probability $10\%$. The initial Gaussian components are sampled from $\mathcal{N}(0,0.7^2\mathbf{I})$. The covariance matrices of the Gaussian transitions are fixed: $0.05^2\mathbf{I}$. The assumptions we explore use the following networks for the mean transitions:
\begin{itemize}
    \item ReLU: Random ReLU networks;
    \item Non-zero: Piece-wise analytic functions with cosine activations, where we force the same function across states if the $L_2$ norm of the conditioned trajectory at time $t$ $(\x_{t-1}, \dots, \x_{t-M})$ is between 3 and 5; and
    \item Zero: Random networks with cosine activations.
\end{itemize}
In the main text, we already indicate the use of locally connected networks \citep{zheng2018dags} where the network dependencies follow a regime-dependent causal structure sampled with a certain sparsity ratio. The networks consist of two-layer MLPs with 16 hidden units. 

\subsection{Training Specifications}

All the experiments are implemented in Pytorch \citep{paszke2017automatic} and performed on NVIDIA RTX 2080Ti GPUs. We use exact batched M-step updates for the discrete distribution parameters, and set a batch size of $500$ (all samples for ECoG data). The batch size is reduced when increasing the number of lags to fit GPU memory requirements. We use Adam optimiser \citep{kingma2014adam} with learning rate $7\cdot 10^{-3}$, and train for a maximum of $100$ epochs ($1000$ for ECoG). We decrease the learning rate by a factor of $0.5$ when likelihood plateaus up to 2 times. Similar to related approaches \citep{halva2020hidden}, we use random restarts to achieve better parameter estimates. The estimated transition means follow the ground-truth structure on the synthetic experiments, and we use $2$-layer MLPs with $32$ hidden units on the ECoG data. We set the covariance matrices independent of the variables.

\subsection{Evaluation}

\paragraph{L2 Distance} Let $d(f,g)$ denote the $L_2$ distance of functions defined as $f:\R^d \rightarrow \R^d$. We compute the $L_2$ distance approximately as follows
\begin{equation}
    d(f, g) := \int_{\x\in\R^d}\sqrt{\big|\big|f(\x) - g(\x) \big|\big|^2}d\x \approx \frac{1}{I}\sum_{i=1}^{I}\sqrt{\big|\big|f\left(\x^{(i)}\right) - g\left(\x^{(i)}\right) \big|\big|^2},
\end{equation}
where in our case, $\x^{(i)}$ denote samples from a held-out dataset with size $1000$. Given we have identifiability up to permutations, we first compute the averaged distances across states $K$ considering all the possible permutations. Then, we select the one which gives lowest error resulting in the following expression:
\begin{equation}\label{eq:error_equation}
\text{err} := \min_{\textbf{k}=\text{perm}\left(\{1,\dots,K\}\right)} \frac{1}{K}\sum_{i=1}^K d(\bm{m}(\cdot,i), \tilde{\bm{m}}(\cdot,k_i)),
\end{equation}
where the samples of the $L_2$ metric are $dM$-dimensional, each of them consisting on $M$ consecutive samples from a $M$th-order MSM sequence. Since accounting for the permutation has a cost of $\mathcal{O}(K!)$, for $K > 5$ we take a greedy approach with cost $\mathcal{O}(K^2)$. Note this alternative returns a suboptimal result if estimation fails.

\paragraph{Causal structure computation}\label{app:jacobi_comp}

Denote the Jacobian of $\tilde{\bm{m}}(\cdot,k)$ at point $(\x_{1}, \dots, \x_{M})$ as $\bm{J}_{\tilde{\bm{m}}(\cdot,k)}(\x_{1}, \dots, \x_{M})$ As mentioned, we compute the regime-dependent causal graph $\tilde{\mathcal{G}}_{1:K}$ via thesholding the Jacobian at each state $k\in\{1,\dots,K\}$. Similar to Eq. (\ref{eq:error_equation}), we use $M$ consecutive  samples from MSM held-out sequences. Then, we classify each group of samples to the corresponding state using the posterior $\p(\s_t|\x_{1:T})$ to ensure the Jacobian captures the effects of the regime of interest. Therefore, we have $K$ sets of variables $\mathcal{X}_{k}, k\in\{1,\dots,K\}$, each with size $|\mathcal{X}_k|=N_k$. Given a sequence, we compute the posterior distribution, and set $\{\x_{t-1}, \dots, \x_{t-M}\}\in\mathcal{X}_k$ if $k=\argmax p(\s_{t}|x_{1:T})$. This yields to the following causal graph estimator $\tilde{\mathbf{G}}_k$ at regime $k\in\{1,\dots,K\}$: 
\begin{equation}
    \tilde{\mathcal{G}}_k := \1\left(  \frac{1}{N_k}\sum_{i=1}^{N_k}\left|\bm{J}_{\tilde{\bm{m}}(\cdot,k)} \left(\x_{1}^{(i)}, \dots, \x_{M}^{(i)}\right)\right| > \tau \right), \quad (\x_{1}^{(i)}, \dots, \x_{M}^{(i)})\in\mathcal{X}_k,
\end{equation}
where we use $\tau=0.05$ in our experiments, and $\1(\cdot)$ denotes the indicator function, which equals to 1 if the argument is true and 0 otherwise. Finally, we compute the averaged $F_1$ score across components with respect to the ground truth causal graph $\mathbf{G}_{1:K}$.

\end{document}